\def\year{2020}\relax
\documentclass[letterpaper]{article} 
\usepackage{aaai20}  
\usepackage{times}  
\usepackage{helvet} 
\usepackage{courier}  
\usepackage[hyphens]{url}  
\usepackage{graphicx} 
\usepackage{graphicx}  
\usepackage{amsmath,amssymb,amsthm}
\usepackage{subfigure}

\title{Discriminative Adversarial Domain Adaptation }
\author{Hui Tang, Kui Jia\thanks{Corresponding author.}\\ 
South China University of Technology\\
eehuitang@mail.scut.edu.cn, kuijia@scut.edu.cn 
}
\begin{document}

\maketitle

\begin{abstract}
Given labeled instances on a source domain and unlabeled ones on a target domain, unsupervised domain adaptation aims to learn a task classifier that can well classify target instances. Recent advances rely on domain-adversarial training of deep networks to learn domain-invariant features. However, due to an issue of mode collapse induced by the separate design of task and domain classifiers, these methods are limited in aligning the joint distributions of feature and category across domains. To overcome it, we propose a novel adversarial learning method termed Discriminative Adversarial Domain Adaptation (DADA). Based on an integrated category and domain classifier, DADA has a novel adversarial objective that encourages a mutually inhibitory relation between category and domain predictions for any input instance. We show that under practical conditions, it defines a minimax game that can promote the joint distribution alignment. Except for the traditional closed set domain adaptation, we also extend DADA for extremely challenging problem settings of partial and open set domain adaptation. Experiments show the efficacy of our proposed methods and we achieve the new state of the art for all the three settings on benchmark datasets.
\end{abstract}

\section{Introduction}

Many machine learning tasks are advanced by large-scale learning of deep models, with image classification \cite{imagenet} as one of the prominent examples. A key factor to achieve such advancements is the availability of massive labeled data on the domains of the tasks of interest. For many other tasks, however, training instances on the corresponding domains are either difficult to collect, or their labeling costs prohibitively. To address the scarcity of labeled data for these \emph{target} tasks/domains, a general strategy is to leverage the massively available labeled data on related \emph{source} ones via domain adaptation \cite{tl_survey}. Even though the source and target tasks share the same label space (i.e. closed set domain adaptation), domain adaptation still suffers from the shift in data distributions. The main objective of domain adaptation is thus to learn domain-invariant features, so that task classifiers learned from the source data can be readily applied to the target domain. In this work, we focus on the unsupervised setting where training instances on the target domain are completely unlabeled.

Recent domain adaptation methods are largely built on modern deep architectures. They rely on great model capacities of these networks to learn hierarchical features that are empirically shown to be more transferable across domains \cite{transferability_theory,metafgnet}. Among them, those based on domain-adversarial training \cite{dann,tada} achieve the current state of the art. Based on the seminal work of DANN \cite{dann}, they typically augment a classification network with an additional domain classifier. The domain classifier takes features from the feature extractor of the classification network as inputs, which is trained to differentiate between instances from the two domains. By playing a minimax game \cite{gans}, adversarial training aims to learn domain-invariant features.

Such domain-adversarial networks can largely reduce the domain discrepancy. However, the separate design of task and domain classifiers has the following shortcomings. Firstly, feature distributions can only be aligned to a certain level, since model capacity of the feature extractor could be large enough to compensate for the less aligned feature distributions. More importantly, given practical difficulties of aligning the source and target distributions with high granularity to the category level (especially for complex distributions with multi-mode structures), the task classifier obtained by minimizing the empirical source risk cannot well generalize to the target data due to an issue of mode collapse \cite{idda,dann_ca}, i.e., the joint distributions of feature and category are not well aligned across the source and target domains.

Recent methods \cite{idda,dann_ca} take the first step to address the above shortcomings by jointly parameterizing the task and domain classifiers into an integrated one. To further push this line, based on such a classifier, we propose a novel adversarial learning method termed \emph{Discriminative Adversarial Domain Adaptation (DADA)}, which encourages a \emph{mutually inhibitory} relation between its domain prediction and category prediction for any input instance, as illustrated in Figure \ref{fig:dada}. This discriminative interaction between category and domain predictions underlies the ability of DADA to reduce domain discrepancy at both the feature and category levels. Intuitively, the adversarial training of DADA mainly conducts competition between the domain neuron (output) and the true category neuron (output). Different from the work \cite{dann_ca} whose mechanism to align the joint distributions is rather implicit, DADA enables explicit alignment between the joint distributions, thus improving the classification of target data. Except for closed set domain adaptation, we also extend DADA for partial domain adaptation \cite{pada}, i.e. the target label space is subsumed by the source one, and open set domain adaptation \cite{bp_for_os}, i.e. the source label space is subsumed by the target one. 
Our main contributions can be summarized as follows.
\begin{itemize}
	\item We propose in this work a novel adversarial learning method, termed DADA, for closed set domain adaptation. Based on an integrated category and domain classifier, DADA has a novel adversarial objective that encourages a \emph{mutually inhibitory} relation between category and domain predictions for any input instance, which can promote the joint distribution alignment across domains.
	
	\item For more realistic partial domain adaptation, we extend DADA by a reliable category-level weighting mechanism, termed DADA-P, which can significantly reduce the negative influence of outlier source instances. 
	
	\item For more challenging open set domain adaptation, we extend DADA by balancing the joint distribution alignment in the shared label space with the classification of outlier target instances, termed DADA-O.
	
	\item Experiments show the efficacy of our proposed methods and we achieve the new state of the art for all the three adaptation settings on benchmark datasets. 
\end{itemize}

\begin{figure*}[t]
	\centering
	\includegraphics[width=0.89\textwidth]{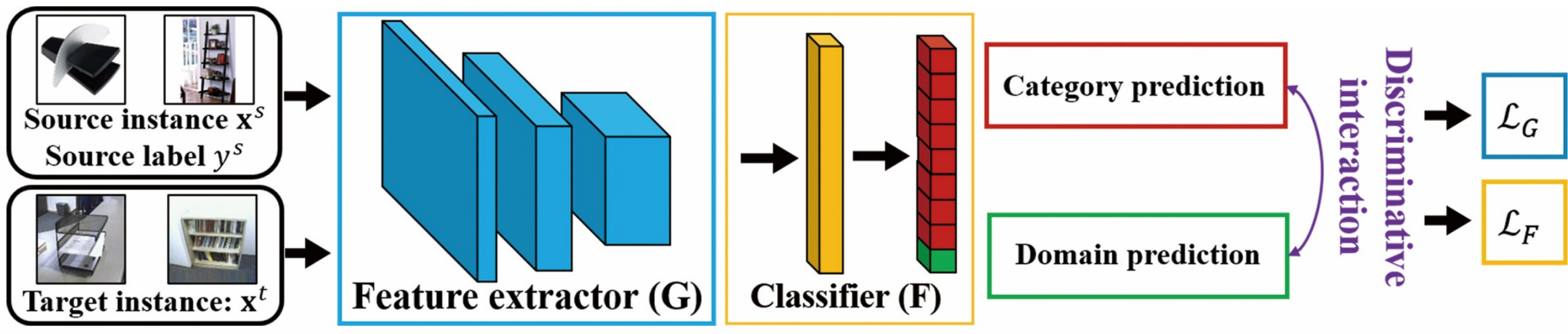}
	\caption{(Best viewed in color.) Discriminative Adversarial Domain Adaptation (DADA), which includes a feature extractor $G(\cdot)$ and an integrated category and domain classifier $F(\cdot)$. The blue and orange colors denote $G(\cdot)$ and $F(\cdot)$, and the losses applied to them, respectively. Note that DADA explicitly establishes a discriminative interaction between category and domain predictions. Please refer to the main text for how the adversarial training objective of DADA is defined.
	}
	\label{fig:dada} 
\end{figure*}

\section{Related Works}
\noindent\textbf{Closed Set Domain Adaptation} After the seminal work of DANN \cite{dann}, ADDA \cite{adda} proposes an untied weight sharing strategy to align the target feature distribution to a fixed source one. SimNet \cite{SimNet} replaces the standard FC-based cross-entropy classifier by a similarity-based one. MADA \cite{mada} and CDAN \cite{cdan} integrate the discriminative category information into domain-adversarial training. VADA \cite{dirt_t} reduces the cluster assumption violation to constrain domain-adversarial training. Some methods \cite{tada,hla} focus on transferable regions to learn domain-invariant features and task classifier. TAT \cite{tat} enhances the discriminability of features to guarantee the adaptability. Some methods \cite{mcd,adr,swd} utilize category predictions from two task classifiers to measure the domain discrepancy. The most related works \cite{idda,dann_ca} to us propose joint parameterization of the task and domain classifiers, which implicitly align the joint distributions. Differently, our proposed DADA makes the joint distribution alignment more explicit, thus promoting classification on the target domain.

\noindent\textbf{Partial Domain Adaptation} The work \cite{iwan} weights each source instance by its importance to the target domain based on one domain classifier, and then trains another domain classifier on target and weighted source instances. The works \cite{san,pada} reduce the contribution of outlier source instances to the task or domain classifiers by utilizing category predictions. Differently, DADA-P weights the proposed source discriminative adversarial loss by a reliable category confidence.

\noindent\textbf{Open Set Domain Adaptation} Previous research \cite{open_set_svm} proposes to reject an instance as the unknown category by threshold filtering. The work \cite{bp_for_os} proposes to utilize adversarial training for both domain adaptation and unknown outlier detection. Differently, DADA-O balances the joint distribution alignment in the shared label space with the outlier rejection.

\section{Method}
Given $\{\left(\mathbf{x}_i^s, y_i^s\right)\}_{i=1}^{n_s}$ of labeled instances sampled from the source domain ${\cal{D}}_s$, and $\{\mathbf{x}_j^t\}_{j=1}^{{n}_t}$ of unlabeled instances sampled from the target domain ${\cal{D}}_t$, the objective of unsupervised domain adaptation is to learn a feature extractor $G(\cdot)$ and a task classifier $C(\cdot)$ such that the expected target risk ${\mathbb{E}}_{(\mathbf{x}^t, y^t)\sim {\cal{D}}_t}[{\cal{L}}_{cls}(C(G(\mathbf{x}^t)), y^t)]$ is low for a certain classification loss function ${\cal{L}}_{cls}(\cdot)$. The domains ${\cal{D}}_s$ and ${\cal{D}}_t$ are assumed to have different distributions. To achieve a low target risk, a typical strategy is to learn $G(\cdot)$ and $C(\cdot)$ by minimizing the sum of the source risk and some notion of \emph{distance} between the source and target domain distributions, inspired by domain adaptation theories \cite{da_theory1,da_theory2}. This strategy is based on a simple rational that the source risk would become a good indicator of the target risk when the distance between the two distributions is getting closer. While most of existing methods use distance measures based on the marginal distributions, it is arguably better to use those based on the joint distributions.

The above strategy is generally implemented by domain-adversarial learning \cite{dann,tada}, where separate task classifier $C(\cdot)$ and domain classifier $D(\cdot)$ are typically stacked on top of the feature extractor $G(\cdot)$.  As discussed before, this type of design has the following shortcomings: (1) model capacity of $G(\cdot)$ could be large enough to make $D(G(\mathbf{x}^s))$ and $D(G(\mathbf{x}^t))$ hardly differentiable for any instance, even though the marginal feature distributions are not well aligned; (2) more importantly, it is difficult to align the source and target distributions with high granularity to the category level (especially for complex distributions with multi-mode structures), and thus $C(\cdot)$ obtained by minimizing the empirical source risk cannot perfectly generalize to the target data due to an issue of mode collapse, i.e. the joint distributions are not well aligned.

To alleviate the above shortcomings, inspired by semi-supervised learning methods based on GANs \cite{gan_tech1,gan_tech2}, the recent work \cite{dann_ca} proposes joint parameterization of $C(\cdot)$ and $D(\cdot)$ into an integrated one $F(\cdot)$. Suppose the classification task of interest has $K$ categories, $F(\cdot)$ is formed simply by augmenting the last FC layer of $C(\cdot)$ with one additional neuron. 

Denote $\mathbf{p}(\mathbf{x}) \in [0, 1]^{K + 1}$ as the output vector of class probabilities of $F(G(\mathbf{x}))$ for an instance $\mathbf{x}$, and $p_k(\mathbf{x})$, $k \in \{1, \dots, K+1\}$, as its $k^{th}$ element. The $k^{th}$ element of the conditional probability vector $\bar{\mathbf{p}}(\mathbf{x})$ is written as follows 
\begin{equation}
\bar{p}_k(\mathbf{x}) = \left\{  
\begin{aligned}
&\frac{p_k(\mathbf{x})}{1-p_{K+1}(\mathbf{x})}, \quad k=1, 2, ..., K  \\
&\qquad\quad 0 \:\qquad, \quad k= K+1
\end{aligned}  
\right. .
\end{equation}
For ease of subsequent notations, we also write $p_k^s = p_k(\mathbf{x}^s)$ and $p_k^t = p_k(\mathbf{x}^t)$. Then, such a network is trained by the classification-aware adversarial learning objective 
\begin{align}
& \min_{F} - \frac{1}{n_s} \sum_{i=1}^{n_s} \log p_{y_i^s}(\mathbf{x}_i^s) - \frac{1}{n_t} \sum_{j=1}^{n_t} \log p_{K+1}(\mathbf{x}_j^t) \label{EqnDANN_CA} \\
& \max_{G} \frac{1}{n_s} \sum_{i=1}^{n_s} \log \bar{p}_{y_i^s}(\mathbf{x}_i^s) + \lambda \frac{1}{n_t} \sum_{j=1}^{n_t} \log (1 - p_{K+1}(\mathbf{x}_j^t)) , \nonumber
\end{align}
where $\lambda$ balances category classification and domain adversarial losses. The mechanism of this objective to align the joint distributions across domains is rather implicit. 

To make it more explicit, based on the integrated classifier $F(\cdot)$, we propose a novel adversarial learning method termed \emph{Discriminative Adversarial Domain Adaptation (DADA)}, which explicitly enables a discriminative interplay of predictions among the domain and $K$ categories for any input instance, as illustrated in Figure \ref{fig:dada}. This discriminative interaction underlies the ability of DADA to promote the joint distribution alignment, as explained shortly.

\subsection{Discriminative Adversarial Learning}
To establish a direct interaction between category and domain predictions, we propose a novel source discriminative adversarial loss that is tailored to the design of the integrated classifier $F(\cdot)$. The proposed loss is inspired by the principle of binary cross-entropy loss. It is written as
\begin{eqnarray}\label{EqnDAdversarialLossOnSource}
\begin{aligned}
{\cal{L}}^s (G, F) = - \frac{1}{n_s} & \sum_{i=1}^{n_s} [\left(1 - p_{K+1}(\mathbf{x}_i^s)\right)\log p_{y_i^s}(\mathbf{x}_i^s) \\ & +  p_{K+1}(\mathbf{x}_i^s)\log\left(1 - p_{y_i^s}(\mathbf{x}_i^s)\right)] .
\end{aligned}
\end{eqnarray}
Intuitively, the proposed loss (\ref{EqnDAdversarialLossOnSource}) establishes a mutually inhibitory relation between $p_{y^s}(\mathbf{x}^s)$ of the prediction on the true category of $\mathbf{x}^s$, and $p_{K+1}(\mathbf{x}^s)$ of the prediction on the domain of $\mathbf{x}^s$. We first discuss how the proposed loss (\ref{EqnDAdversarialLossOnSource}) works during adversarial training, and we show that under practical conditions, minimizing (\ref{EqnDAdversarialLossOnSource}) over the classifier $F(\cdot)$ has the effects of discriminating among task categories while distinguishing the source domain from the target one, and maximizing (\ref{EqnDAdversarialLossOnSource}) over the feature extractor $G(\cdot)$ can discriminatively align the source domain to the target one.

\noindent\textbf{Discussion} We first write the gradient formulas of ${\cal{L}}^s$ on any source instance $\mathbf{x}^s$ w.r.t. $p^s_{y^s}$ and $p^s_{K+1}$ as
\begin{align}
\nabla_{p_{y^s}^s} = \frac{\partial {\cal{L}}^s}{\partial p^s_{y^s}} \nonumber = \frac{p^s_{y^s}p^s_{K+1} - (1-p^s_{y^s})(1-p^s_{K+1})}{p^s_{y^s}(1 - p^s_{y^s})}, \nonumber
\end{align}
\begin{align}
\nabla_{p_{K+1}^s} = \frac{\partial {\cal{L}}^s}{\partial p^s_{K+1}} = \log \frac{p^s_{y^s}}{1-p^s_{y^s}}. \nonumber
\end{align}
Since both $p_{y^s}^s$ and $p_{K+1}^s$ are among the $K+1$ output probabilities of the classifier $F(G(\mathbf{x}^s))$, we always have $p_{y^s}^s \leq 1 - p_{K+1}^s$ and $p_{K+1}^s \leq 1 - p_{y^s}^s$, suggesting $\nabla_{p_{y^s}^s} \leq 0$. When the loss (\ref{EqnDAdversarialLossOnSource}) is minimized over $F(\cdot)$ via stochastic gradient descent (SGD), we have the update $p_{y^s}^s \leftarrow p_{y^s}^s - \eta \nabla_{p_{y^s}^s}$ where $\eta$ is the learning rate, and since $\nabla_{p_{y^s}^s} \leq 0$, $p_{y^s}^s$ increases; when it is maximized over $G(\cdot)$ via stochastic gradient ascent (SGA), we have the update $p_{y^s}^s \leftarrow p_{y^s}^s + \eta \nabla_{p_{y^s}^s}$, and since $\nabla_{p_{y^s}^s} \leq 0$,  $p_{y^s}^s$ decreases. Then, we discuss the change of $p_{K+1}^s$ in two cases: (1) in case of $p_{y^s}^s > 0.5$ that guarantees $\nabla_{p_{K+1}^s}>0$, when minimizing the loss (\ref{EqnDAdversarialLossOnSource}) over $F(\cdot)$ by SGD update $p_{K+1}^s \leftarrow p_{K+1}^s - \eta \nabla_{p_{K+1}^s}$, we have decreased $p_{K+1}^s$, and when maximizing it over $G(\cdot)$ by SGA update $p_{K+1}^s \leftarrow p_{K+1}^s + \eta \nabla_{p_{K+1}^s}$, we have increased $p_{K+1}^s$; (2) in case of $p_{y^s}^s < 0.5$ that guarantees $\nabla_{p_{K+1}^s}<0$, when minimizing the loss (\ref{EqnDAdversarialLossOnSource}) over $F(\cdot)$ by SGD update, we have increased $p_{K+1}^s$, and when maximizing it over $G(\cdot)$ by SGA update, we have decreased $p_{K+1}^s$, as shown in Figure \ref{fig:see_change}.

\begin{figure}[!ht]
	\centering
	\includegraphics[width=0.95\columnwidth]{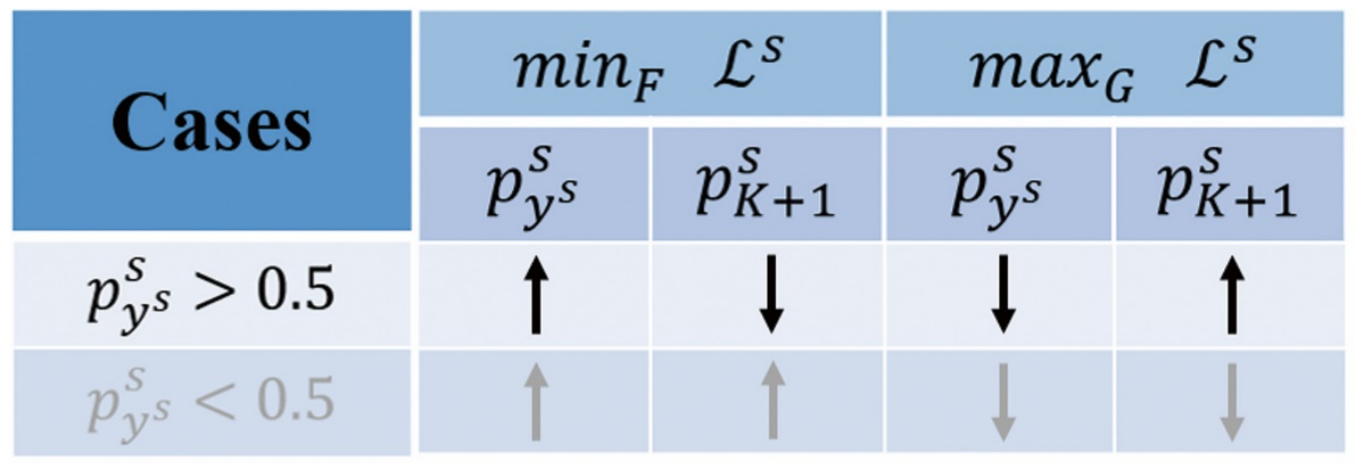}
	\caption{Changes of $p_{y^s}^s$ and $p_{K+1}^s$ when minimizing and maximizing the loss (\ref{EqnDAdversarialLossOnSource}) in the two cases.}
	\label{fig:see_change}
\end{figure}

For discriminative adversarial domain adaptation, we expect that (1) when minimizing the proposed loss (\ref{EqnDAdversarialLossOnSource}) over $F(\cdot)$, task categories of the source domain is discriminative and the source domain is distinctive from the target one, which can be achieved when $p_{y^s}^s$ increases and $p_{K+1}^s$ decreases; (2) when maximizing it over $G(\cdot)$, the source domain is aligned to the target one while retains discriminability, which can be achieved when $p_{y^s}^s$ decreases and $p_{K+1}^s$ increases in the case of $p_{y^s}^s>0.5$. To meet the expectations, the condition of $p_{y^s}^s>0.5$ for all source instances should be always satisfied. This is practically achieved by pre-training DADA on the labeled source data using a $K$-way cross-entropy loss, and maintaining in the adversarial training of DADA the same supervision signal. We present in the supplemental material empirical evidence on benchmark datasets that shows the efficacy of our used scheme. 

To achieve the joint distribution alignment, the explicit interplay between category and domain predictions for any target instance should also be created. Motivated by recent works \cite{mada,cdan} which alleviate the issue of mode collapse by aligning each instance to several most related categories, we propose a target discriminative adversarial loss based on the design of the integrated classifier $F(\cdot)$, by using the conditional category probabilities to weight the domain predictions. It is written as 
\begin{eqnarray}\label{EqnDAdversarialLossOnTarget}
\begin{aligned}
&{\cal{L}}^t_{F} (G, F) =  - \frac{1}{n_t} \sum_{j=1}^{n_t} \sum_{k=1}^K \bar{p}_k(\mathbf{x}_j^t) \log \hat{p}_{K+1}^k (\mathbf{x}_j^t) \\
&{\cal{L}}^t_{G} (G, F) =  \frac{1}{n_t} \sum_{j=1}^{n_t} \sum_{k=1}^K \bar{p}_k(\mathbf{x}_j^t) \log (1-\hat{p}_{K+1}^k (\mathbf{x}_j^t)), 
\end{aligned}
\end{eqnarray}
where the ${k'}^{th}$ element of the domain prediction vector $\hat{\mathbf{p}}^{k}$ for the $k^{th}$ category is written as follows
\begin{equation}
\hat{p}_{k'}^{k}(\mathbf{x}) = \left\{  
\begin{aligned}
&\frac{p_{k'}(\mathbf{x})}{p_{k}(\mathbf{x})+p_{K+1}(\mathbf{x})}, \quad k'=k, K+1  \\
&\qquad\quad 0 \qquad\qquad, \quad {\rm otherwise}
\end{aligned}  
\right. .
\end{equation}
An intuitive explanation for our proposed (\ref{EqnDAdversarialLossOnTarget}) is provided in the supplemental material.

Established knowledge from cluster analysis \cite{fano_ineq} indicates that we can estimate clusters with a low probability of error only if the conditional entropy is small. To this end, we adopt the entropy minimization principle \cite{em}, which is written as 
\begin{eqnarray}\label{EqnEMlLoss}
{\cal{L}}_{em}^t (G, F) = \frac{1}{n_t} \sum^{n_t}_{j=1} {\cal{H}}(\bar{\mathbf{p}}(\mathbf{x}^t_j)),
\end{eqnarray}
where ${\cal{H}}(\cdot)$ computes the entropy of a probability vector. Combining (\ref{EqnDAdversarialLossOnSource}), (\ref{EqnDAdversarialLossOnTarget}), and (\ref{EqnEMlLoss}) gives the following minimax problem of our proposed DADA 
\begin{eqnarray}\label{EqnOverallLoss}
\begin{aligned}
&\min_{F} {\cal{L}}_F = \lambda ({\cal{L}}^s + {\cal{L}}^t_{F}) - {\cal{L}}_{em}^t \\
&\max_{G} {\cal{L}}_G = \lambda ({\cal{L}}^s + {\cal{L}}^t_{G}) - {\cal{L}}_{em}^t ,
\end{aligned}
\end{eqnarray}
where $\lambda$ is a hyper-parameter that trade-offs the adversarial domain adaptation objective with the entropy minimization one in the unified optimization problem. Note that in the minimization problem of (\ref{EqnOverallLoss}), ${\cal{L}}_{em}^t$ serves as a regularizer for learning $F(\cdot)$ to avoid the trivial solution (i.e. all instances are assigned to the same category), and in the maximization problem of (\ref{EqnOverallLoss}), it helps learn more target-discriminative features, which can alleviate the negative effect of adversarial feature adaptation on the adaptability \cite{tat}.

By optimizing (\ref{EqnOverallLoss}), the joint distribution alignment can be enhanced. This ability comes from the better use of discriminative information from both the source and target domains. Concretely, DADA constrains the domain classifier so that it clearly/explicitly knows the classification boundary, thus reducing false alignment between different categories. By deceiving such a strong domain classifier, DADA can learn a feature extractor that better aligns the two domains. \emph{We also theoretically prove in the supplemental material that DADA can better bound the expected target error.}

\subsection{Extension for Partial Domain Adaptation}
Partial domain adaptation is a more realistic setting, where the target label space is subsumed by the source one. The false alignment between the outlier source categories and the target domain is unavoidable. To address it, existing methods \cite{san,iwan,pada} utilize the category or domain predictions, to decrease the contribution of source outliers to the training of task or domain classifiers. Inspired by these ideas, we extend DADA for partial domain adaptation by using a reliable category-level weighting mechanism, which is termed DADA-P. 

Concretely, we average the conditional probability vectors $\bar{\mathbf{p}}(\mathbf{x}^t) \in [0, 1]^K$ over all target data and then normalize the averaged vector $\bar{\mathbf{c}} \in [0, 1]^K$ by dividing its largest element. The category weight vector $\mathbf{c} \in [0, 1]^K$ with $c_k$ as its $k^{th}$ element is derived by a convex combination of the normalized vector and an all-ones vector $\mathbf{1}$, as follows 
\begin{eqnarray}\label{EqnAvgCateWeightOverTarget}
\begin{aligned}
\bar{\mathbf{c}} &= \frac{1}{n_t} \sum_{j=1}^{n_t} \bar{\mathbf{p}}(\mathbf{x}_j^t) \\
\mathbf{c} &= \lambda \frac{\bar{\mathbf{c}}}{\max(\bar{\mathbf{c}})} + (1 - \lambda) \mathbf{1}, 
\end{aligned}
\end{eqnarray}
where $\lambda \in [0,1]$ is to suppress the detection noise of outlier source categories in the early stage of training. Then, we apply the category weight vector $\mathbf{c}$ to the proposed discriminative adversarial loss for any source instance, leading to 
\begin{align}
{\cal{L}}^s (G, F) = - \frac{1}{n_s} & \sum_{i=1}^{n_s} c_{y_i^s}[\left(1 - p_{K+1}(\mathbf{x}_i^s)\right)\log p_{y_i^s}(\mathbf{x}_i^s) \nonumber \\ & + p_{K+1}(\mathbf{x}_i^s)\log\left(1 - p_{y_i^s}(\mathbf{x}_i^s)\right)]. \label{EqnWDAdversarialLossOnSource}
\end{align}

Since predicted probabilities on the outlier source categories are more likely to increase when minimizing $-{\cal{L}}_{em}^t$ over $F(\cdot)$, which incurs negative transfer. To avoid it, we minimize ${\cal{L}}_{em}^t$ over $F(\cdot)$ and the objective of DADA-P is 
\begin{eqnarray}\label{EqnPartialOverallLoss}
\begin{aligned}
&\min_{F} {\cal{L}}_F = \lambda ({\cal{L}}^s + {\cal{L}}^t_{F}) + {\cal{L}}_{em}^t \\
&\max_{G} {\cal{L}}_G = \lambda ({\cal{L}}^s + {\cal{L}}^t_{G}) - {\cal{L}}_{em}^t. 
\end{aligned}
\end{eqnarray}
By optimizing it, DADA-P can simultaneously alleviate negative transfer and promote the joint distribution alignment across domains in the shared label space.

\subsection{Extension for Open Set Domain Adaptation}
Open set domain adaptation is a very challenging setting, where the source label space is subsumed by the target one. We denominate the shared category and all unshared categories between the two domains as the ``known category" and ``unknown category" respectively. The goal of open set domain adaptation is to correctly classify any target instance as the known or unknown category. The false alignment between the known and unknown categories is inevitable. To this end, the work \cite{bp_for_os} proposes to make a pseudo decision boundary for the unknown category, which enables the feature extractor to reject some target instances as outliers. Inspired by this work, we extend DADA for open set domain adaptation by training the classifier to classify all target instances as the unknown category with a small probability $q$, which is termed DADA-O. Assuming the predicted probability on the unknown category as the $K^{th}$ element of $\mathbf{p}(\mathbf{x}^t)$, i.e., $p_K(\mathbf{x}^t)$, the modified target adversarial loss when minimized over the integrated classifier $F(\cdot)$ is 
\begin{eqnarray}\label{EqnDomainAdversarialLossOnTargetForOpenSet}
\begin{aligned}
&{\cal{L}}^t_{F} (G, F) = \\ & - \frac{1}{n_t} \sum_{j=1}^{n_t} q \log p_K(\mathbf{x}_j^t) - (1 - q) \log p_{K+1}(\mathbf{x}_j^t), 
\end{aligned}
\end{eqnarray}
where $0 < q < 0.5$. When maximized over the feature extractor $G(\cdot)$, we still use the discriminative loss ${\cal{L}}_G^t$ in (\ref{EqnDAdversarialLossOnTarget}). Replacing ${\cal{L}}_F^t$ in (\ref{EqnOverallLoss}) with (\ref{EqnDomainAdversarialLossOnTargetForOpenSet}) gives the overall adversarial objective of DADA-O, which can achieve a balance between domain adaptation and outlier rejection.

We utilize all target instances to obtain the concept of ``unknown'', which is very helpful for the classification of unknown target instances as the unknown category but can cause the misclassification of known target instances as the unknown category. This issue can be alleviated by selecting an appropriate $q$. If $q$ is too small, the unknown target instances cannot be correctly classified; if $q$ is too large, the known target instances can be misclassified. By choosing an appropriate $q$, the feature extractor can separate the unknown target instances from the known ones while aligning the joint distributions in the shared label space.

\section{Experiments}

\subsection{Datasets and Implementation Details}
\noindent\textbf{Office-31} \cite{office31} is a popular benchmark domain adaptation dataset consisting of $4,110$ images of $31$ categories collected from three domains: Amazon (\textbf{A}), Webcam (\textbf{W}), and DSLR (\textbf{D}). We evaluate on six settings.

\noindent\textbf{Syn2Real} \cite{visda} is the largest benchmark. Syn2Real-C has over $280K$ images of $12$ shared categories in the combined training, validation, and testing domains. The $152,397$ images on the training domain are synthetic ones by rendering 3D models. The validation and test domains comprise real images, and the validation one has $55,388$ images. We use the training domain as the source domain and validation one as the target domain. For partial domain adaptation, we choose images of the first $6$ categories (in alphabetical order) in the validation domain as the target domain and form the setting: \textbf{Synthetic 12} $\rightarrow$ \textbf{Real 6}. For open set domain adaptation, we evaluate on Syn2Real-O, which includes two domains. The training/synthetic domain uses synthetic images from the $12$ categories of Syn2Real-C as ``known''. The validation/real domain uses images of the $12$ categories from the validation domain of Syn2Real-C as ``known'', and $50$k images from $69$ other categories as ``unknown''. We use the training and validation domains of Syn2Real-O as the source and target domains respectively. 

\noindent\textbf{Implementation Details} We follow standard evaluation protocols for unsupervised domain adaptation \cite{dann,tada}: we use all labeled source and all unlabeled target instances as the training data. For all tasks of Office-31 and \textbf{Synthetic 12} $\rightarrow$ \textbf{Real 6}, based on ResNet-50 \cite{resnet}, we report the classification result on the target domain of mean($\pm$standard deviation) over three random trials. For other tasks of Syn2Real, we evaluate the accuracy of each category based on ResNet-101 and ResNet-152 (for closed and open set domain adaptation respectively). For each base network, we use all its layers up to the second last one as the feature extractor $G(\cdot)$, and set the neuron number of its last FC layer as $K+1$ to have the integrated classifier $F(\cdot)$. Exceptionally, we follow the work \cite{visda} and replace the last FC layer of ResNet-152 with three FC layers of 512 neurons. All base networks are pre-trained on ImageNet \cite{imagenet}. We firstly pre-train them on the labeled source data, and then fine-tune them on both the labeled source data and unlabeled target data via adversarial training, where we maintain the same supervision signal as the pre-training. 

We follow DANN \cite{dann} to use the SGD training schedule: the learning rate is adjusted by $\eta_p=\frac{\eta_0}{(1+\alpha p)^\beta}$, where $p$ denotes the process of training iterations that is normalized to be in $[0, 1]$, and we set $\eta_0 = 0.0001$, $\alpha = 10$, and $\beta = 0.75$; the hyper-parameter $\lambda$ is initialized at $0$ and is gradually increased to $1$ by $\lambda_{p}=\frac{2}{1+\exp(-\gamma p)}-1$, where we set $\gamma = 10$. We empirically set $q=0.1$. We implement all our methods by \textbf{PyTorch}. The code will be available at https://github.com/huitangtang/DADA-AAAI2020.

\begin{table*}[!t]
	\caption{Ablation studies using Office-31 based on ResNet-50. Please refer to the main text for how they are defined.}
	\label{table:alter_base_office31} 
	\begin{center}
		\resizebox{0.9\textwidth}{!}{
			\begin{tabular}{lccccccc}
				\hline
				Methods                & A $\rightarrow$ W & D $\rightarrow$ W & W $\rightarrow$ D & A $\rightarrow$ D & D $\rightarrow$ A & W $\rightarrow$ A & Avg \\
				\hline
				No Adaptation          & 79.9$\pm$0.3 & 96.8$\pm$0.4 & 99.5$\pm$0.1 & 84.1$\pm$0.4 & 64.5$\pm$0.3 & 66.4$\pm$0.4 & 81.9 \\
				
				DANN       & 81.2$\pm$0.3 & 98.0$\pm$0.2 & 99.8$\pm$0.0 & 83.3$\pm$0.3 & 66.8$\pm$0.3 & 66.1$\pm$0.3 & 82.5 \\
				
				DANN-CA    & 85.4$\pm$0.4 & 98.2$\pm$0.2 & 99.8$\pm$0.0 & 87.1$\pm$0.4 & 68.5$\pm$0.2 & 67.6$\pm$0.3 & 84.4 \\
				\hline
				DADA (w/o em + w/o td)         &  91.0$\pm$0.2 & 98.7$\pm$0.1 & \textbf{100.0}$\pm$0.0 & 90.8$\pm$0.2 & 70.9$\pm$0.3 & 70.2$\pm$0.3 & 86.9 \\
				
				DADA (w/o em)                  & 91.8$\pm$0.1 & 99.0$\pm$0.1 & \textbf{100.0}$\pm$0.0 & 92.5$\pm$0.3 & 72.8$\pm$0.2 & 72.3$\pm$0.3 & 88.1 \\
				
				DADA               & \textbf{92.3}$\pm$0.1 & \textbf{99.2}$\pm$0.1 & \textbf{100.0}$\pm$0.0 & \textbf{93.9}$\pm$0.2 & \textbf{74.4}$\pm$0.1 & \textbf{74.2}$\pm$0.1 & \textbf{89.0} \\
				
				\hline
			\end{tabular}
		}
	\end{center}
\end{table*}

\begin{table*}[!ht]
	\caption{Results for closed set domain adaptation on Office-31 based on ResNet-50. Note that SimNet is implemented by an \textbf{unknown} framework; MADA and DANN-CA are implemented by \textbf{Caffe}; all the other methods are implemented by \textbf{PyTorch}.
	}
	\label{table:results_office31}
	\begin{center}
		\resizebox{0.9\textwidth}{!}{
			\begin{tabular}{lccccccc}
				\hline
				Methods                & A $\rightarrow$ W & D $\rightarrow$ W & W $\rightarrow$ D & A $\rightarrow$ D & D $\rightarrow$ A & W $\rightarrow$ A & Avg \\
				\hline
				No Adaptation \cite{resnet}         & 79.9$\pm$0.3 & 96.8$\pm$0.4 & 99.5$\pm$0.1 & 84.1$\pm$0.4 & 64.5$\pm$0.3 & 66.4$\pm$0.4 & 81.9 \\
				
				DAN \cite{dan}          & 81.3$\pm$0.3 & 97.2$\pm$0.0 & 99.8$\pm$0.0 & 83.1$\pm$0.2 & 66.3$\pm$0.0 & 66.3$\pm$0.1 & 82.3 \\
				
				DANN \cite{dann}      & 81.2$\pm$0.3 & 98.0$\pm$0.2 & 99.8$\pm$0.0 & 83.3$\pm$0.3 & 66.8$\pm$0.3 & 66.1$\pm$0.3 & 82.5 \\
				
				ADDA \cite{adda}                & 86.2$\pm$0.5 & 96.2$\pm$0.3 & 98.4$\pm$0.3 & 77.8$\pm$0.3 & 69.5$\pm$0.4 & 68.9$\pm$0.5 & 82.9 \\
				
				MADA \cite{mada}                   & 90.0$\pm$0.1 & 97.4$\pm$0.1 & 99.6$\pm$0.1 & 87.8$\pm$0.2 & 70.3$\pm$0.3 & 66.4$\pm$0.3 & 85.2 \\
				
				VADA \cite{dirt_t}              & 86.5$\pm$0.5 & 98.2$\pm$0.4 & 99.7$\pm$0.2 & 86.7$\pm$0.4 & 70.1$\pm$0.4 & 70.5$\pm$0.4 & 85.4 \\
				
				DANN-CA \cite{dann_ca}   & 91.35 & 98.24 & 99.48 & 89.94 & 69.63 & 68.76 & 86.2 \\
				
				GTA \cite{gen_to_adapt} & 89.5$\pm$0.5 & 97.9$\pm$0.3 & 99.8$\pm$0.4 & 87.7$\pm$0.5 & 72.8$\pm$0.3 & 71.4$\pm$0.4 & 86.5 \\
				
				MCD \cite{mcd}                  & 88.6$\pm$0.2 & 98.5$\pm$0.1 & \textbf{100.0}$\pm$0.0 & 92.2$\pm$0.2 & 69.5$\pm$0.1 & 69.7$\pm$0.3 & 86.5 \\
				
				CDAN+E \cite{cdan}  & \textbf{94.1}$\pm$0.1 & 98.6$\pm$0.1 & \textbf{100.0}$\pm$0.0 & 92.9$\pm$0.2 & 71.0$\pm$0.3 & 69.3$\pm$0.3 & 87.7 \\
				
				TADA \cite{tada}                & 94.3$\pm$0.3 & 98.7$\pm$0.1 & 99.8$\pm$0.2 & 91.6$\pm$0.3 & 72.9$\pm$0.2 & 73.0$\pm$0.3 & 88.4 \\
				
				SymNets \cite{symnets}          & 90.8$\pm$0.1 & 98.8$\pm$0.3 & \textbf{100.0}$\pm$0.0 & \textbf{93.9}$\pm$0.5 & \textbf{74.6}$\pm$0.6 & 72.5$\pm$0.5 & 88.4 \\
				
				TAT \cite{tat}                  & 92.5$\pm$0.3 & \textbf{99.3}$\pm$0.1 & \textbf{100.0}$\pm$0.0 & 93.2$\pm$0.2 & 73.1$\pm$0.3 & 72.1$\pm$0.3 & 88.4 \\
				\hline
				\textbf{DADA}              & 92.3$\pm$0.1 & 99.2$\pm$0.1 & \textbf{100.0}$\pm$0.0 & \textbf{93.9}$\pm$0.2 & 74.4$\pm$0.1 & \textbf{74.2}$\pm$0.1 & \textbf{89.0} \\
				\hline
			\end{tabular}
		}
	\end{center}
\end{table*}

\subsection{Analysis} 

\noindent\textbf{Ablation Study} We conduct ablation studies on Office-31 to investigate the effects of key components of our proposed DADA based on ResNet-50. Our ablation studies start with the very baseline termed ``No Adaptation'' that simply fine-tunes a ResNet-50 on the source data. To validate the mutually inhibitory relation enabled by DADA, we use DANN \cite{dann} and DANN-CA \cite{dann_ca} respectively as the second and third baselines. To investigate how the entropy minimization principle helps learn more target-discriminative features, we remove the entropy minimization loss (\ref{EqnEMlLoss}) from our main minimax problem (\ref{EqnOverallLoss}), denoted as ``DADA (w/o em)''. To know effects of the proposed source and target discriminative adversarial losses (\ref{EqnDAdversarialLossOnSource}) and (\ref{EqnDAdversarialLossOnTarget}), we remove both (\ref{EqnEMlLoss}) and (\ref{EqnDAdversarialLossOnTarget}) from (\ref{EqnOverallLoss}), denoted as ``DADA (w/o em + w/o td)''. 

\begin{figure}[t]
	\centering
	\includegraphics[width=0.9\columnwidth]{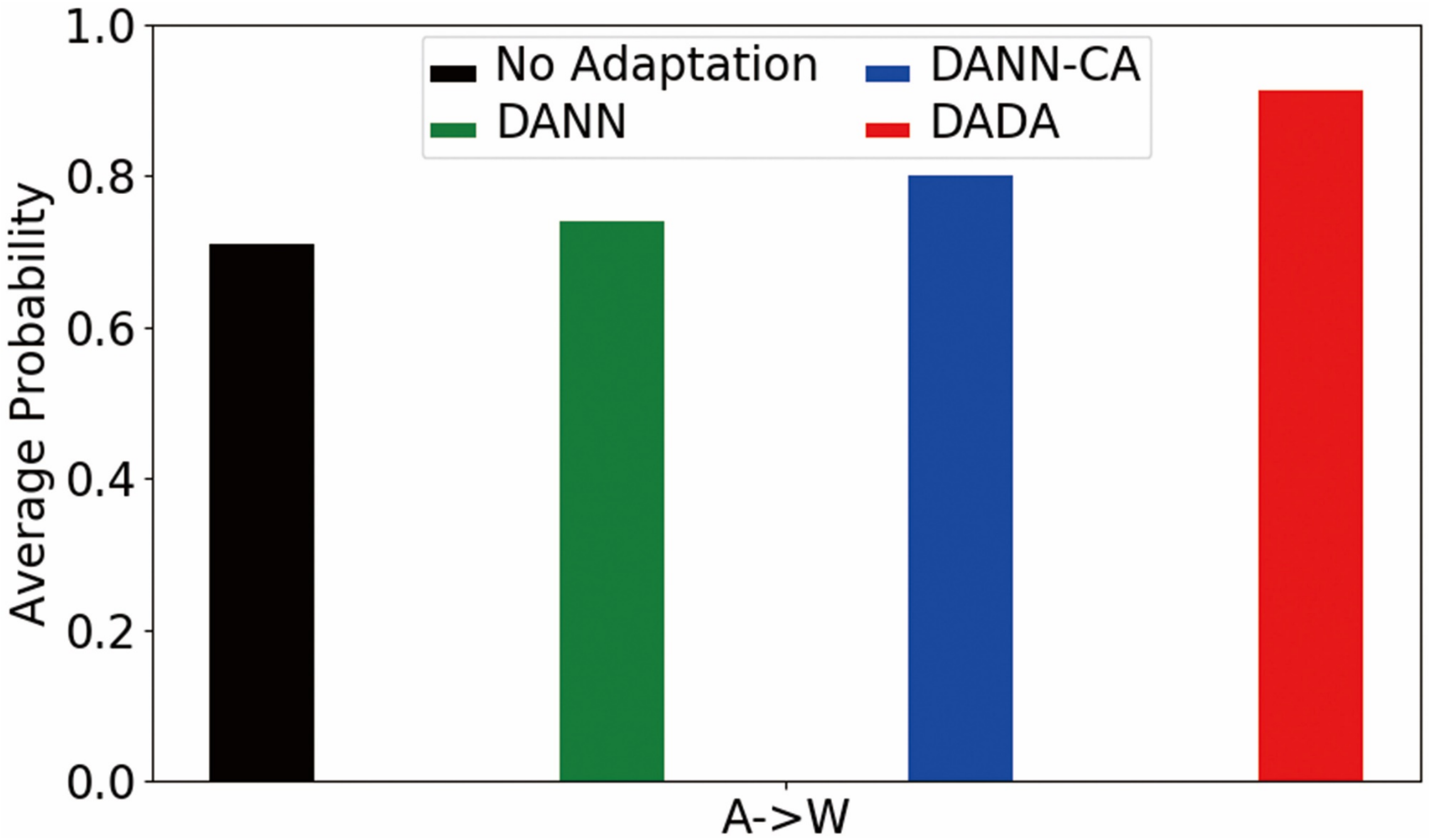}
	\caption{Average probability on the true category over all target instances by task classifiers of different methods. 
	}
	\label{fig:avg_prob}
\end{figure}

Results in Table \ref{table:alter_base_office31} show that although DANN improves over ``No Adaptation'', its result is much worse than DANN-CA, verifying the efficacy of the design of the integrated classifier $F(\cdot)$. ``DADA (w/o em + w/o td)'' improves over DANN-CA and ``DADA (w/o em)'' improves over ``DADA (w/o em + w/o td)'', showing the efficacy of our proposed discriminative adversarial learning. DADA significantly outperforms DANN and DANN-CA, confirming the efficacy of the proposed mutually inhibitory relation between the category and domain predictions in aligning the joint distributions of feature and category across domains. Table \ref{table:alter_base_office31} also confirms that entropy minimization is helpful to learn more target-discriminative features.

\begin{table*}[!ht]
	\caption{Results for closed set domain adaptation on Syn2Real-C based on ResNet-101. Note that all compared methods are based on \textbf{PyTorch} implementation.}
	\label{table:results_visda} 
	\begin{center}
		\resizebox{0.92\textwidth}{!}{
			\begin{tabular}{lccccccccccccc}
				\hline
				Methods                & plane & bcycl & bus & car & horse & knife & mcycl & person & plant & sktbrd & train & truck & mean \\
				\hline
				
				No Adaptation \cite{resnet}           & 55.1 & 53.3 & 61.9 & 59.1 & 80.6 & 17.9 & 79.7 & 31.2 & 81.0 & 26.5 & 73.5 & 8.5 & 52.4 \\
				
				DANN \cite{dann}  & 81.9 & 77.7 & 82.8 & 44.3 & 81.2 & 29.5 & 65.1 & 28.6 & 51.9 & 54.6 & 82.8 & 7.8 & 57.4 \\
				
				DAN \cite{dan}         & 87.1 & 63.0 & 76.5 & 42.0 & 90.3 & 42.9 & 85.9 & 53.1 & 49.7 & 36.3 & 85.8 & 20.7 & 61.1 \\
				
				MCD \cite{mcd} & 87.0 & 60.9 & \textbf{83.7} & 64.0 & 88.9 & 79.6 & 84.7 & \textbf{76.9} & 88.6 & 40.3 & 83.0 & 25.8 & 71.9 \\
				
				GPDA \cite{gpda}  & 83.0 & 74.3 & 80.4 & \textbf{66.0} & 87.6 & 75.3 & 83.8 & 73.1 & \textbf{90.1} & 57.3 & 80.2 & 37.9 & 73.3 \\
				
				ADR \cite{adr}    & 87.8 & \textbf{79.5} & \textbf{83.7} & 65.3 & \textbf{92.3} & 61.8 & \textbf{88.9} & 73.2 & 87.8 & 60.0 & 85.5 & 32.3 & 74.8 \\
				
				\hline
				\textbf{DADA} & \textbf{92.9} & 74.2 & 82.5 & 65.0 & 90.9 & \textbf{93.8} & 87.2 & 74.2 & 89.9 & \textbf{71.5} & \textbf{86.5} & \textbf{48.7} & \textbf{79.8} \\
				\hline
			\end{tabular}
		}
	\end{center} 
\end{table*}

\begin{table*}[!ht]
	\begin{center}
		\caption{Results for open set domain adaptation on Syn2Real-O based on ResNet-152. \emph{Known} indicates the mean classification result over the known categories whereas \emph{Mean} also includes the unknown category. The table below shows the results when the Known-to-Unknown Ratio in the target domain is set to $1:10$. All compared methods are based on \textbf{PyTorch} implementation.}
		\label{table:results_visda_open_set}
		\resizebox{0.92\textwidth}{!}{
			\begin{tabular}{lccccccccccccccc}
				\hline
				\multicolumn{16}{c}{Known-to-Unknown Ratio $= 1:1$} \\
				\hline
				Methods                          & plane & bcycl & bus & car & horse & knife & mcycl & person & plant & sktbrd & train & truck & unk & Known & Mean\\
				\hline
				
				No Adaptation \cite{resnet}      & 49 & 20 & 29 & 47 & 62 & 27 & 79 & 3 & 37 & 19 & 70 & 1 & 62 & 36 & 38 \\
				
				DAN \cite{dan}                   & 51 & 40 & 42 & 56 & 68 & 24 & 75 & 2 & 39 & 30 & 71 & 2 & 75 & 41 & 44 \\
				
				DANN \cite{dann}                 & 59 & 41 & 16 & 54 & 77 & 18 & 88 & 4 & 44 & 32 & 68 & 4 & 61 & 42 & 43 \\
				
				AODA \cite{bp_for_os}            & 85 & 71 & 65 & 53 & \textbf{83} & 10 & 79 & 36 & \textbf{73} & 56 & 79 & \textbf{32} & \textbf{87} & 60 & 62 \\
				\hline
				\textbf{DADA-O}                   & \textbf{88} & \textbf{76} & \textbf{76} & \textbf{64} & 79 & \textbf{46} & \textbf{91} & \textbf{62} & 52 & \textbf{63} & \textbf{86} & 8 & 55 & \textbf{66} & \textbf{65} \\
				\hline
				\multicolumn{16}{c}{Known-to-Unknown Ratio $= 1:10$} \\
				\hline
				AODA \cite{bp_for_os}            & \textbf{80} & \textbf{63} & 59 & 63 & \textbf{83} & 12 & 89 & 5 & \textbf{61} & 14 & 79 & 0 & \textbf{69} & 51 & 52 \\
				\hline
				\textbf{DADA-O}                   & 77 & \textbf{63} & \textbf{75} & \textbf{71} & 38 & \textbf{33} & \textbf{92} & \textbf{58} & 47 & \textbf{50} & \textbf{89} & \textbf{1} & 50 & \textbf{58} & \textbf{57} \\
				\hline
			\end{tabular}
		}
	\end{center}
\end{table*}

\noindent\textbf{Quantitative Comparison} To compare the efficacy of different methods in reducing domain discrepancy at the category level, we visualize the average probability on the true category over all target instances by task classifiers of No Adaptation, DANN, DANN-CA, and DADA on \textbf{A} $\rightarrow$ \textbf{W} in Figure \ref{fig:avg_prob}. Note that here we use labels of the target data for the quantization of category-level domain discrepancy. Figure \ref{fig:avg_prob} shows that our proposed DADA gives the predicted probability on the true category of any target instance a better chance to approach $1$, meaning that target instances are more likely to be correctly classified by DADA, i.e., a better category-level domain alignment.

\begin{table}[!htb]
	\begin{center}
		\caption{Results for partial domain adaptation on Syn2Real-C based on ResNet-50. Note that all compared methods are based on \textbf{PyTorch} implementation.}
		\label{table:results_visda_partial_transfer}
		\resizebox{0.85\columnwidth}{!}{
			\begin{tabular}{lc}
				\hline
				Methods  & Synthetic 12$\rightarrow$Real 6  \\
				\hline
				No Adaptation \cite{resnet} & 45.26  \\
				
				DAN \cite{dan}              & 47.60   \\
				
				DANN \cite{dann}            & 51.01   \\
				
				RTN \cite{rtn}              & 50.04   \\  
				
				PADA \cite{pada}            & 53.53  \\
				\hline
				\textbf{DADA-P}              & \textbf{69.06}  \\
				\hline
			\end{tabular}
		}
	\end{center}
\end{table}

\subsection{Results}
\noindent\textbf{Closed Set Domain Adaptation} We compare in Tables \ref{table:results_office31} and \ref{table:results_visda} our proposed method with existing ones on Office-31 and Syn2Real-C based on ResNet-50 and ResNet-101 respectively. Whenever available, results of existing methods are quoted from their respective papers or the recent works \cite{mada,cdan,tat,mcd}. Our proposed DADA outperforms existing methods, testifying the efficacy of DADA in aligning the joint distributions of feature and category across domains.

\noindent\textbf{Partial Domain Adaptation} We compare in Table \ref{table:results_visda_partial_transfer} our proposed method to existing ones on Syn2Real-C based on ResNet-50. Results of existing methods are quoted from the work \cite{pada}. Our proposed DADA-P substantially outperforms all comparative methods by $+15.53\%$, showing the effectiveness of DADA-P on reducing the negative influence of source outliers while promoting the joint distribution alignment in the shared label space. 

\noindent\textbf{Open Set Domain Adaptation} We compare in Table \ref{table:results_visda_open_set} our proposed method with existing ones on Syn2Real-O based on ResNet-152. Results of existing methods are quoted from the recent work \cite{visda}. Our proposed DADA-O outperforms all comparative methods in both evaluation metrics of Known and Mean, showing the efficacy of DADA-O in both aligning joint distributions of the known instances and identifying the unknown target instances. It is noteworthy that DADA-O improves over the state-of-the-art method AODA by a large margin when the known-to-unknown ratio in the target domain is much smaller than $1$, i.e. the false alignment between the known source and unknown target instances will be much more serious. This observation confirms the efficacy of DADA-O.

We provide more results and analysis for the three problem settings in the supplemental material.

\section{Conclusion}
We propose a novel adversarial learning method termed Discriminative Adversarial Domain Adaptation (DADA) to overcome the limitation in aligning the joint distributions of feature and category across domains, which is due to an issue of mode collapse induced by the separate design of task and domain classifiers. Based on an integrated task and domain classifier, DADA has a novel adversarial objective that encourages a mutually inhibitory relation between the category and domain predictions, which can promote the joint distribution alignment. Unlike previous methods, DADA explicitly enables a discriminative interaction between category and domain predictions. Except for closed set domain adaptation, we also extend DADA for more challenging problem settings of partial and open set domain adaptation. Experiments on benchmark datasets testify the efficacy of our proposed methods for all the three settings.

\section{Acknowledgments}
This work is supported in part by the National Natural Science Foundation of China (Grant No.: 61771201), the Program for Guangdong Introducing Innovative and Enterpreneurial Teams (Grant No.: 2017ZT07X183), and the Guangdong R\&D key project of China (Grant No.: 2019B010155001).

\fontsize{9.0pt}{10.0pt} \selectfont 
\bibliography{egbib}
\bibliographystyle{aaai}

\clearpage

\appendix

\newtheorem{proposition}{Proposition}
\numberwithin{equation}{section}
\newtheorem{theorem}{Theorem}
\newtheorem{definition}{Definition}

\setcounter{secnumdepth}{2}

We provide an intuitive explanation for our proposed loss (4) in Section \ref{sec2}. We theoretically prove that our proposed method can better bound the expected target error than existing ones in Section \ref{sec3}. We provide more results and analysis on benchmark datasets of Digits, Office-31, Office-Home, and ImageNet-Caltech for closed set, partial, and open set domain adaptation in Section \ref{sec4}. We present empirical evidence on benchmark datasets of digits that shows the efficacy of our used training scheme in Section \ref{sec5}. We will release the code soon.

\section{Intuitive Explanation for Our Proposed Loss (4)}
\label{sec2}
We denote the output vector of class scores of $F(G(\mathbf{x}))$ before the final softmax operation for an instance $\mathbf{x}$ as $\mathbf{o}(\mathbf{x}) \in \mathbb{R}^{K+1}$, and its $k^{th}$ element as $o_k(\mathbf{x})$, $k \in \{1,\dots, K+1\}$. We denote the output vector of class probabilities of $F(G(\mathbf{x}))$ after the final softmax operation for an instance $\mathbf{x}$ as $\mathbf{p}(\mathbf{x}) \in [0,1]^{K+1}$, and its $k^{th}$ element as $p_k(\mathbf{x})$, $k \in \{1,\dots, K+1\}$. We write $p_k(\mathbf{x})$, $k \in \{1,\dots, K+1\}$ as 
\begin{eqnarray}\label{EqnSoftmax}
\begin{aligned}
p_k(\mathbf{x}) = \frac{\exp(o_k(\mathbf{x}))}{\sum_{k'=1}^{K+1}\exp(o_{k'}(\mathbf{x}))}. 
\end{aligned}
\end{eqnarray}
We always have $\sum_{k=1}^{K+1} p_k(\mathbf{x}) = 1$ for any instance $\mathbf{x}$. When maximized over the feature extractor $G(\cdot)$, the adversarial loss on an unlabeled target instance $\mathbf{x}^t$ (cf. objective (2) in Section \textbf{Discriminative Adversarial Domain Adaptation} in the paper) is written as 
\begin{eqnarray}\label{EqnTarAdvLoss}
\begin{aligned}
{l}^t(G, F) &  = \log(1-p_{K+1}(\mathbf{x}^t)) 
\\ & = \log(\sum_{k=1}^K p_k(\mathbf{x}^t)) 
\\ & = \log\Big(\frac{\sum_{k=1}^{K} \exp(o_k(\mathbf{x}^t))}{\sum_{k'=1}^{K+1} \exp(o_{k'}(\mathbf{x}^t))}\Big). 
\end{aligned}
\end{eqnarray}
We write the gradient formulas of ${l}^t$ w.r.t. $o_k(\mathbf{x})$, $k \in \{1, \dots, K\}$ as 
\begin{align}
&\nabla_{o_k(\mathbf{x}^t)} \nonumber \\ & = \frac{\partial {l}^t}{\partial o_{k}(\mathbf{x}^t)} \nonumber
\\ & = \frac{\sum_{k'=1}^{K+1} \exp(o_{k'}(\mathbf{x}^t))}{\sum_{k=1}^{K} \exp(o_k(\mathbf{x}^t))} \cdot \frac{\exp(o_k(\mathbf{x}^t))\cdot\exp(o_{K+1}(\mathbf{x}^t))}{\Big(\sum_{k'=1}^{K+1} \exp(o_{k'}(\mathbf{x}^t))\Big)^2}, 
\end{align}
where $\nabla_{o_k(\mathbf{x}^t)}$, $k \in \{1, \dots, K\}$ differ in the term of $\exp(o_k(\mathbf{x}^t))$, meaning that they are proportional to the class scores of $o_k(\mathbf{x}^t)$. In other words, the higher the class score is (i.e., the higher the class probability is), the stronger gradient the corresponding category neuron back-propagates, suggesting that the target instance is aligned to several most confident/related categories on the source domain. Such a mechanism to align the joint distributions of feature and category across domains is rather implicit. To make it more explicit, our proposed target discriminative adversarial loss (cf. loss (4) in Section \textbf{Discriminative Adversarial Learning} in the paper) uses the conditional probabilities to weight the category-wise domain predictions. By such a design, the discriminative adversarial training on the target data explicitly conducts the competition between the domain neuron (output) and the most confident category neuron (output) as the discriminative adversarial training on the source data does, thus promoting the category-level domain alignment. This is what we mean by the \emph{mutually inhibitory relation between the category and domain predictions} for any input instance.

This intuitive explanation manifests that the adversarial training of DADA clearly and explicitly utilizes the discriminative information of the target domain, thus improving the alignment of joint distributions of feature and category across domains. 

\section{Generalization Error Analysis for Our Proposed DADA}
\label{sec3}
We prove that our proposed DADA can better bound the expected target error than existing domain adaptation methods \cite{dann,adda,mada,SimNet,iCAN,dirt_t,cdan,tada,hla,dann_ca}, taking the similar formalism of theoretical results of domain adaptation \cite{da_theory1,da_theory2}.

For all hypothesis spaces introduced below, we assume them of finite effective size, i.e., finite VC dimension, so that the following distance measures defined over these spaces can be estimated from finite instances \cite{da_theory2}. 
We consider a fixed representation function $G(\cdot)$ from the instance set $X$ to the feature space $Z$, i.e., $\mathbf{z} = G(\mathbf{x})$, and a hypothesis space $\cal{H}$ for the $K$-category task classifier $C(\cdot)$ from the feature space $Z$ to the label space $Y$, i.e., $C \in \cal{H}$ \cite{dann}. Note that $\mathbf{y} \in Y$ is the $K$-dimensional one-hot vector for any label $y$. Denote the marginal feature distribution and the joint distribution of feature and category by $P_{Z}^s$ and $P_{Z, Y}^s$ for the source domain ${\cal{D}}_s$, and similarly $P_{Z}^t$ and $P_{Z, Y}^t$ for the target domain ${\cal{D}}_t$, respectively. 
Let $\epsilon_s(C) = \mathbb{E}_{(\mathbf{z}, \mathbf{y}) \sim P_{Z, Y}^s} {\rm I}[C(\mathbf{z}) \ne \mathbf{y}]$ be the expected source error of a hypothesis $C \in \cal{H}$ w.r.t. the joint distribution $P_{Z, Y}^s$, where ${\rm I}[a]$ is the indicator function which is $1$ if predicate $a$ is true, and $0$ otherwise. Similarly, $\epsilon_t(C) = \mathbb{E}_{(\mathbf{z}, \mathbf{y}) \sim P_{Z, Y}^t} {\rm I}[C(\mathbf{z}) \ne \mathbf{y}]$ denotes the expected target error of $C$ w.r.t. the joint distribution $P_{Z, Y}^t$. Let $C^* = {\rm argmin}_{C \in \cal{H}}[\epsilon_s(C)+\epsilon_t(C)]$ be the ideal joint hypothesis that explicitly embodies the notion of adaptability \cite{da_theory2}. Let $\epsilon_s(C, C^*) = \mathbb{E}_{(\mathbf{z}, \mathbf{y}) \sim P_{Z, Y}^s} {\rm I} [C(\mathbf{z}) \ne C^*(\mathbf{z})]$ and $\epsilon_t(C, C^*) = \mathbb{E}_{(\mathbf{z}, \mathbf{y}) \sim P_{Z, Y}^t} {\rm I} [C(\mathbf{z}) \ne C^*(\mathbf{z})]$ be the disagreement between hypotheses $C$ and $C^*$ w.r.t. the joint distributions $P_{Z, Y}^s$ and $P_{Z, Y}^t$ respectively. Specified by the two works \cite{da_theory1,da_theory2}, the probabilistic bound of the expected target error $\epsilon_t(C)$ of the hypothesis $C$ is given by the sum of the expected source error $\epsilon_s(C)$, the combined error $[\epsilon_s(C^*) + \epsilon_t(C^*)]$ of the ideal joint hypothesis $C^*$, and the distribution discrepancy across data domains, as the follow
\begin{eqnarray}\label{EqnExpTarErrBound}
\begin{aligned}
&\epsilon_t(C) \le \\ 
&\epsilon_s(C) + [\epsilon_s(C^*) + \epsilon_t(C^*)] + |\epsilon_s(C, C^*) - \epsilon_t(C, C^*)|. 
\end{aligned}
\end{eqnarray}

For domain adaptation to be possible, a natural assumption is that there exists the ideal joint hypothesis $C^* \in \cal{H}$ so that the combined error $[\epsilon_s(C^*) + \epsilon_t(C^*)]$ is small. The ideal joint hypothesis $C^*$ may not be unique, since in practice we always have the same error obtained by two different machine learning models. Denote a set of ideal joint hypotheses by ${\cal{H}}^*$, which is a subset of $\cal{H}$, i.e., ${\cal{H}}^* \subset \cal{H}$. Based on this assumption, domain adaptation aims to reduce the domain discrepancy $|\epsilon_s(C, C^*) - \epsilon_t(C, C^*)|$. Let $\mathbf{c} = C(\mathbf{z})$ be the proxy of the label vector $\mathbf{y}$ of $\mathbf{z}$, for every pair of $(\mathbf{z}, \mathbf{y}) \sim P_{Z, Y}^s \cup P_{Z, Y}^t$. Denote the thus obtained proxies of the joint distributions $P_{Z, Y}^s$ and $P_{Z, Y}^t$ by $P_{Z, C}^s=(\mathbf{z}, C(\mathbf{z}))_{\mathbf{z} \sim P_{Z}^s}$ and $P_{Z, C}^t=(\mathbf{z}, C(\mathbf{z}))_{\mathbf{z} \sim P_{Z}^t}$, respectively \cite{jointDistOptimal}. Then, by definition, $\epsilon_s(C, C^*) = \mathbb{E}_{(\mathbf{z}, \mathbf{y}) \sim P_{Z, Y}^s} {\rm I} [C(\mathbf{z}) \ne C^*(\mathbf{z})] = \mathbb{E}_{(\mathbf{z}, \mathbf{c}) \sim P_{Z, C}^s} {\rm I} [\mathbf{c} \ne C^*(\mathbf{z}))]$, and similarly $\epsilon_t(C, C^*) = \mathbb{E}_{(\mathbf{z}, \mathbf{y}) \sim P_{Z, Y}^t} {\rm I} [C(\mathbf{z}) \ne C^*(\mathbf{z})] = \mathbb{E}_{(\mathbf{z}, \mathbf{c}) \sim P_{Z, C}^t} {\rm I} [ \mathbf{c} \ne C^*(\mathbf{z})]$. Based on the two joint distribution proxies, we have the domain discrepancy 
\begin{eqnarray}\label{EqnExpDistDiscBound0}
\begin{aligned}
& |\epsilon_s(C, C^*) - \epsilon_t(C, C^*)| \\
& = |\mathbb{E}_{(\mathbf{z}, \mathbf{y}) \sim P_{Z, Y}^s} {\rm I} [C(\mathbf{z}) \ne C^*(\mathbf{z})] \\& \qquad - \mathbb{E}_{(\mathbf{z}, \mathbf{y}) \sim P_{Z, Y}^t} {\rm I} [C(\mathbf{z}) \ne C^*(\mathbf{z})]| \\
& = |\mathbb{E}_{(\mathbf{z}, \mathbf{c}) \sim P_{Z, C}^s} {\rm I} [\mathbf{c} \ne C^*(\mathbf{z}))] \\& \qquad - \mathbb{E}_{(\mathbf{z}, \mathbf{c}) \sim P_{Z, C}^t} {\rm I} [ \mathbf{c} \ne C^*(\mathbf{z})]|. 
\end{aligned}
\end{eqnarray}

Inspired by the two works \cite{cdan,da_theory3}, we next introduce four definitions of the distance measure that can upper bound the domain discrepancy.

\begin{definition}
	Let ${\cal{F}}_{{\cal{H}}^*} = \{F(C^*(\mathbf{z}), \mathbf{c} )={\rm I} [\mathbf{c} \ne C^*(\mathbf{z})] | C^* \in {\cal{H}}^* \}$ be a (loss) difference hypothesis space over the joint variable of $(C^*(\mathbf{z}), \mathbf{c})$, where $F:(C^*(\mathbf{z}), \mathbf{c}) \mapsto \{0, 1\}$ computes the empirical 0-1 classification loss of the task classifier $C^* \in {\cal{H}}^*$ for any input pair of $(\mathbf{z}, \mathbf{c}) \sim P_{Z, C}^s \cup P_{Z, C}^t$. Then, the ${\cal{F}}_{{\cal{H}}^*}$-distance between two distributions $P_{Z, C}^s$ and $P_{Z, C}^t$, is defined as 
	\begin{eqnarray}\label{EqnExpDistDiscBound1}
	\begin{aligned}
	& d_{{\cal{F}}_{{\cal{H}}^*}} (P_{Z, C}^s, P_{Z, C}^t) \\ 
	& \triangleq \sup_{F \in {\cal{F}}_{{\cal{H}}^*}, C^* \in {\cal{H}}^*} |\mathbb{E}_{(\mathbf{z}, \mathbf{c}) \sim P_{Z, C}^s} F(C^*(\mathbf{z}), \mathbf{c}) \\ & \qquad\qquad\qquad\qquad - \mathbb{E}_{(\mathbf{z}, \mathbf{c}) \sim P_{Z, C}^t} F(C^*(\mathbf{z}), \mathbf{c})| \\
	& = \sup_{C^* \in {\cal{H}}^*} |\mathbb{E}_{(\mathbf{z}, \mathbf{c}) \sim P_{Z, C}^s} {\rm I} [\mathbf{c} \ne C^*(\mathbf{z}))] \\ & \qquad\qquad\quad - \mathbb{E}_{(\mathbf{z}, \mathbf{c}) \sim P_{Z, C}^t} {\rm I} [ \mathbf{c} \ne C^*(\mathbf{z})]|. 
	\end{aligned}
	\end{eqnarray}
\end{definition}

\begin{definition}
	Let $\cal{F}$ be a (loss) difference hypothesis space, which contains a class of functions $F:(\mathbf{z},\mathbf{c}) \mapsto \{0,1\}$ over the joint variable of $(\mathbf{z}, \mathbf{c}) \sim P_{Z, C}^s \cup P_{Z, C}^t$. Then, the $\cal{F}$-distance between two distributions $P_{Z, C}^s$ and $P_{Z, C}^t$, is defined as 
	\begin{eqnarray}\label{EqnExpDistDiscBound2}
	\begin{aligned}
	& d_{\cal{F}} (P_{Z, C}^s, P_{Z, C}^t) \\ 
	& \triangleq \sup_{F \in \cal{F}} |\mathbb{E}_{(\mathbf{z}, \mathbf{c}) \sim P_{Z, C}^s} F(\mathbf{z}, \mathbf{c}) - \mathbb{E}_{(\mathbf{z}, \mathbf{c}) \sim P_{Z, C}^t} F(\mathbf{z}, \mathbf{c})|. 
	\end{aligned}
	\end{eqnarray}
\end{definition}

\begin{definition}
	Let ${\cal{F}}_{\cal{H}} = \{F:(C'(\mathbf{z}), \mathbf{c}) \mapsto \{0, 1\} | C' \in {\cal{H}} \}$ be a (loss) difference hypothesis space over the joint variable of $(C'(\mathbf{z}), \mathbf{c})$, where $F(C'(\mathbf{z}), \mathbf{c})$ computes the empirical 0-1 classification loss of the task classifier $C' \in {\cal{H}}$ for any input pair of $(\mathbf{z}, \mathbf{c}) \sim P_{Z, C}^s \cup P_{Z, C}^t$. Then, the ${\cal{F}}_{\cal{H}}$-distance between two distributions $P_{Z, C}^s$ and $P_{Z, C}^t$, is defined as 
	\begin{eqnarray}\label{EqnExpDistDiscBound3}
	\begin{aligned}
	& d_{{\cal{F}}_{\cal{H}}} (P_{Z, C}^s, P_{Z, C}^t) \\ 
	& \triangleq \sup_{F \in {\cal{F}}_{\cal{H}}, C' \in \cal{H}} |\mathbb{E}_{(\mathbf{z}, \mathbf{c}) \sim P_{Z, C}^s} F(C'(\mathbf{z}), \mathbf{c}) \\ & \qquad\qquad\qquad\quad - \mathbb{E}_{(\mathbf{z}, \mathbf{c}) \sim P_{Z, C}^t} F(C'(\mathbf{z}), \mathbf{c})|.  
	\end{aligned}
	\end{eqnarray}
\end{definition}

\begin{definition}
	Let $\cal{D}$ be a (loss) difference hypothesis space, which contains a class of functions $D : \mathbf{z} \mapsto \{ 0, 1 \}$ over $\mathbf{z} \sim P_Z^s \cup P_Z^t$. Then, the $\cal{D}$-distance between two distributions $P_{Z, C}^s$ and $P_{Z, C}^t$, is defined as 
	\begin{eqnarray}\label{EqnExpDistDiscBound4}
	\begin{aligned}
	d_{\cal{D}} (P_{Z}^s, P_{Z}^t) \triangleq \sup_{D \in \cal{D}} |\mathbb{E}_{\mathbf{z} \sim P_{Z}^s} D(\mathbf{z}) - \mathbb{E}_{\mathbf{z} \sim P_{Z}^t} D(\mathbf{z})|. 
	\end{aligned}
	\end{eqnarray}
\end{definition}

We are now ready to give an upper bound on the domain discrepancy in terms of the distance measures we have defined.

\begin{theorem}
	\label{tighterBound}
	The distribution discrepncy between the source and target domains $|\epsilon_s(C, C^*) - \epsilon_t(C, C^*)|$ can be upper bounded by the ${\cal{F}}_{{\cal{H}}^*}$-distance, the ${\cal{F}}_{\cal{H}}$-distance, the $\cal{F}$-distance, and the $\cal{D}$-distance as follows 
	\begin{eqnarray}\label{EqnExpDistDiscBound5}
	\begin{aligned}
	&|\epsilon_s(C, C^*) - \epsilon_t(C, C^*)| \\
	&\le d_{{\cal{F}}_{{\cal{H}}^*}} (P_{Z, C}^s, P_{Z, C}^t) \\
	&\le d_{{\cal{F}}_{\cal{H}}} (P_{Z, C}^s, P_{Z, C}^t) \\
	&\le d_{\cal{F}} (P_{Z, C}^s, P_{Z, C}^t) \\
	&\le d_{\cal{D}} (P_{Z}^s, P_{Z}^t). 
	\end{aligned}
	\end{eqnarray}
\end{theorem}

\begin{proof}
	Comparing (\ref{EqnExpDistDiscBound0}) and (\ref{EqnExpDistDiscBound1}), since $|\mathbb{E}_{(\mathbf{z}, \mathbf{c}) \sim P_{Z, C}^s} {\rm I} [\mathbf{c} \ne C^*(\mathbf{z}))] - \mathbb{E}_{(\mathbf{z}, \mathbf{c}) \sim P_{Z, C}^t} {\rm I} [ \mathbf{c} \ne C^*(\mathbf{z})]| \le \sup_{C^* \in {\cal{H}}^*} |\mathbb{E}_{(\mathbf{z}, \mathbf{c}) \sim P_{Z, C}^s} {\rm I} [\mathbf{c} \ne C^*(\mathbf{z}))] - \mathbb{E}_{(\mathbf{z}, \mathbf{c}) \sim P_{Z, C}^t} {\rm I} [ \mathbf{c} \ne C^*(\mathbf{z})]|$, we have $|\epsilon_s(C, C^*) - \epsilon_t(C, C^*)| \le d_{{\cal{F}}_{{\cal{H}}^*}} (P_{Z, C}^s, P_{Z, C}^t)$.
	
	Since by definition the hypothesis space $\cal{F}$ contains all functions that map $(\mathbf{z}, \mathbf{c})$ to $\{0, 1\}$, $F(C^*(\mathbf{z}), \mathbf{c})$ is also a function in $\cal{F}$ that can be written as the form of functions in ${\cal{F}}_{{\cal{H}}^*}$. The hypothesis space ${\cal{F}}_{{\cal{H}}^*}$ is subsumed by $\cal{F}$ , i.e., ${\cal{F}}_{{\cal{H}}^*} \subset \cal{F}$. Thus, we have $|\epsilon_s(C, C^*) - \epsilon_t(C, C^*)| \le d_{{\cal{F}}_{{\cal{H}}^*}} (P_{Z, C}^s, P_{Z, C}^t) \le d_{\cal{F}} (P_{Z, C}^s, P_{Z, C}^t)$.
	
	Similarly, since ${\cal{F}}_{\cal{H}} \subset \cal{F}$, we have $d_{{\cal{F}}_{\cal{H}}} (P_{Z, C}^s, P_{Z, C}^t) \le d_{\cal{F}} (P_{Z, C}^s, P_{Z, C}^t)$. Since by definition the ideal joint hypothesis set ${\cal{H}}^* \subset \cal{H}$, the hypothesis space ${\cal{F}}_{{\cal{H}}^*}$ is subsumed by ${\cal{F}}_{\cal{H}}$, i.e., ${\cal{F}}_{{\cal{H}}^*} \subset {\cal{F}}_{\cal{H}}$. Thus, we have $|\epsilon_s(C, C^*) - \epsilon_t(C, C^*)| \le d_{{\cal{F}}_{{\cal{H}}^*}} (P_{Z, C}^s, P_{Z, C}^t) \le d_{{\cal{F}}_{\cal{H}}} (P_{Z, C}^s, P_{Z, C}^t)$. 
	
	Since by definition the hypothesis space $\cal{D}$ contains all functions that map $\mathbf{z}$ to $\{0, 1\}$, $F(\mathbf{z}, \mathbf{c})=F(\mathbf{z}, C(\mathbf{z}))$ is also a function in $\cal{D}$ that can be written as the form of functions in $\cal{F}$. The hypothesis space $\cal{F}$ is subsumed by $\cal{D}$, i.e., $ \cal{F} \subset \cal{D}$. Thus, we have $d_{\cal{F}} (P_{Z, C}^s, P_{Z, C}^t) \le d_{\cal{D}} (P_{Z}^s, P_{Z}^t)$.
	
	These prove the inequality $|\epsilon_s(C, C^*) - \epsilon_t(C, C^*)| \le d_{{\cal{F}}_{{\cal{H}}^*}} (P_{Z, C}^s, P_{Z, C}^t) \le d_{{\cal{F}}_{\cal{H}}} (P_{Z, C}^s, P_{Z, C}^t) \le d_{\cal{F}} (P_{Z, C}^s, P_{Z, C}^t) \le d_{\cal{D}} (P_{Z}^s, P_{Z}^t)$.
\end{proof}

Theorem \ref{tighterBound} shows that the ${\cal{F}}_{{\cal{H}}^*}$-distance can best upper bound the domain discrepncy $|\epsilon_s(C, C^*) - \epsilon_t(C, C^*)|$, but cannot be computable, since instances on the target domain for unsupervised domain adaptation are unlabeled; the ${\cal{F}}_{\cal{H}}$-distance can better bound the domain discrepncy $|\epsilon_s(C, C^*) - \epsilon_t(C, C^*)|$ than the $\cal{F}$-distance and the $\cal{D}$-distance, and the hypothesis space ${\cal{F}}_{\cal{H}}$ can be implemented by conditioning the function $F(\mathbf{z}, \mathbf{c}) \in \cal{F}$ on the other one $C(\mathbf{z}) \in \cal{H}$; the $\cal{F}$-distance can tighter bound the domain discrepncy $|\epsilon_s(C, C^*) - \epsilon_t(C, C^*)|$ than the $\cal{D}$-distance, and the hypothesis space $\cal{F}$ can be realized by taking as input both the feature representation $\mathbf{z}$ and the category prediction $\mathbf{c}$; the $\cal{D}$-distance can loosely bound the domain discrepncy $|\epsilon_s(C, C^*) - \epsilon_t(C, C^*)|$, and the hypothesis space $\cal{D}$ can be instantiated by taking as input only the feature representation $\mathbf{z}$. Since existing deep domain adaptation methods are based on deep neuron networks, the inference of the hypothesis space ${\cal{F}}_{{\cal{H}}^*} \subset {\cal{F}}_{\cal{H}} \subset \cal{F} \subset \cal{D}$ is reasonable and realistic in that, for any given function, there must exist a feedforward neural network or multilayer perceptron, which can approximate it with arbitrarily small error \cite{univAppTheo1,univAppTheo2}, however, the effective model capacity is limited by the capabilities of the optimization algorithm \cite{dlbook}. 

Since these methods \cite{dann,adda,SimNet,iCAN,dirt_t} are based on a separate domain classifier that takes as input only the feature representation, they aim to measure and minimize the $\cal{D}$-distance. Since these methods \cite{mada,cdan,tada,hla} are based on one or several conditional domain classifiers that take as input both the feature representation and the category prediction, they aim to measure and minimize the $\cal{F}$-distance. Since the recent work \cite{dann_ca} and the proposed DADA unify the task and domain classifiers into an integrated one, i.e., conditioning the domain classifier on the task classifier, they aim to measure and minimize the ${\cal{F}}_{\cal{H}}$-distance. The ${\cal{F}}_{\cal{H}}$-distance can be upper bounded by the optimal solution of the integrated domain and task classifier $F(\cdot)$. 
In the meanwhile, the upper bound of ${\cal{F}}_{\cal{H}}$-distance is minimized by learning a domain-invariant feature extractor $G(\cdot)$. 

Furthermore, our proposed DADA can be intuitively formalized as category-regularized domain-adversarial training, since our proposed discriminative adversarial training can learn an integrated classifier $F(\cdot)$ that has explicit intra-domain discrimination and inter-domain indistinguishability, which may enable a better performed ideal joint hypothesis $C^*$. Consequently, the expected target error $\epsilon_t(C)$ can be better approximated by the expected source error $\epsilon_s(C)$. As verified above, our proposed DADA can formally better bound the expected target error than existing domain adaptation methods.

\section{Additional Results and Analysis}
\label{sec4}
\subsection{Datasets}
\noindent\textbf{Digits} datasets of MNIST \cite{mnist}, Street View House Numbers (SVHN) \cite{svhn}, and USPS \cite{usps} are popular. we follow ADR \cite{adr} and evaluate on three adaptation settings of \textbf{SVHN}$\rightarrow$\textbf{MNIST}, \textbf{MNIST}$\rightarrow$\textbf{USPS}, and \textbf{USPS}$\rightarrow$\textbf{MNIST}. For all adaptation settings, we adopt the same network architecture and experimental setting as ADR.

\noindent\textbf{Office-31} \cite{office31} is a benchmark domain adaptation dataset as introduced in Section \textbf{Datasets and Implementation Details} in the paper. For partial domain adaptation, we select images of $10$ categories shared by Office-31 and Caltech-256 \cite{caltech256} in each domain of Office-31 as the target domain. Note that the source domain here contains $31$ categories and the target domain here contains $10$ categories. For open set domain adaptation, we use the selected $10$ categories as the known categories. In alphabetical order, $11-20$ categories and $21-31$ categories are used as the unknown categories in the source and target domains respectively. In this setting, an $11$-category classification is performed. 

\noindent\textbf{Office-Home} \cite{office_home} is a much more challenging benchmark dataset for domain adaptation, which includes $15,500$ images of $65$ object categories in office and home scenes, shared by four extremely distinct domains: Artistic images (\textbf{Ar}), Clip Art (\textbf{Cl}), Product images (\textbf{Pr}), and Real-World images (\textbf{Rw}). We build $12$ adaptation settings: \textbf{Ar} $\rightarrow$ \textbf{Cl}, \textbf{Ar} $\rightarrow$ \textbf{Pr}, \textbf{Ar} $\rightarrow$ \textbf{Rw}, \textbf{Cl} $\rightarrow$ \textbf{Ar}, \textbf{Cl} $\rightarrow$ \textbf{Pr}, \textbf{Cl} $\rightarrow$ \textbf{Rw}, \textbf{Pr} $\rightarrow$ \textbf{Ar}, \textbf{Pr} $\rightarrow$ \textbf{Cl}, \textbf{Pr} $\rightarrow$ \textbf{Rw}, \textbf{Rw} $\rightarrow$ \textbf{Ar}, \textbf{Rw} $\rightarrow$ \textbf{Cl}, \textbf{Rw} $\rightarrow$ \textbf{Pr}. For partial domain adaptation, we choose images of the first $25$ categories (in alphabetical order) in each domain of this dataset as target domains. Note that each source domain here contains $65$ categories and each target domain here contains $25$ categories.

\noindent\textbf{ImageNet-Caltech} is built from ImageNet \cite{imagenet} that contains $1000$ categories, and Caltech-256 \cite{caltech256} that contains $256$ categories. They share $84$ common categories, thus we construct two adaptation settings: \textbf{I (1000)} $\rightarrow$ \textbf{C (84)}, and \textbf{C (256)} $\rightarrow$ \textbf{I (84)}. When ImageNet is used as the source domain, we use its training set; when it is used as the target domain, we use its validation set to prevent the model from the effect of pre-training on its training set. 

\subsection{Closed Set Domain Adaptation. }
\subsubsection{Effect of the $\lambda$}
We provide the empirical evidence on Office-31 \cite{office31} based on ResNet-50 \cite{resnet} for the effect of the hyper-parameter $\lambda$ on keeping the source instances satisfying the condition of $p_{y^s}^s>0.5$ (cf. Section \textbf{Discriminative Adversarial Learning} in the paper for its derivation) in the early stage of adversarial training of DADA in Figure \ref{fig:verify_lambda}, which shows that the rate of source instances failing to satisfy the condition rises rapidly in the early stage of adversarial training when the $\lambda$ is not used.

\begin{table*}[t]
	\caption{Comparison on Office-31 based on ResNet-50 with an alternative choice of adversarial loss for target instances. Please refer to the main text for how this alternative is defined.}
	\label{table:other_variants_office31}
	\begin{center}
		\begin{tabular}{lccccccc}
			\hline
			Methods     & A $\rightarrow$ W & D $\rightarrow$ W & W $\rightarrow$ D & A $\rightarrow$ D & D $\rightarrow$ A & W $\rightarrow$ A & Avg \\
			\hline
			
			DADA-DC     & 90.4$\pm$0.1 & 98.7$\pm$0.1 & \textbf{100.0}$\pm$0.0 & 92.5$\pm$0.3 & 72.5$\pm$0.2 & 73.0$\pm$0.3 & 87.9 \\
			
			DADA        & \textbf{92.3}$\pm$0.1 & \textbf{99.2}$\pm$0.1 & \textbf{100.0}$\pm$0.0 & \textbf{93.9}$\pm$0.2 & \textbf{74.4}$\pm$0.1 & \textbf{74.2}$\pm$0.1 & \textbf{89.0} \\
			
			\hline
		\end{tabular}
	\end{center}
\end{table*}

\begin{table*}
	\caption{Analysis of robustness for different methods on benchmark datasets of MNIST \cite{mnist}, SVHN \cite{svhn}, and USPS \cite{usps} based on modified LeNet.}
	\label{table:results_digits} 
	\begin{center}
		\begin{tabular}{lcccc}
			\hline
			Methods                 & SVHN $\rightarrow$ MNIST & MNIST $\rightarrow$ USPS & USPS $\rightarrow$ MNIST & Avg \\
			\hline
			No Adaptation           & 67.1 & 77.0 & 68.1 & 70.7 \\ 
			
			DDC \cite{ddc}          & 68.1$\pm$0.3 & 79.1$\pm$0.5 & 66.5$\pm$3.3 & 71.2 \\ 
			
			DANN \cite{dann}        & 73.9 & 77.1$\pm$1.8 & 73.0$\pm$0.2 & 74.7 \\ 
			
			DRCN \cite{drcn}        & 82.0$\pm$0.1 & 91.8$\pm$0.09 & 73.7$\pm$0.04 & 82.5 \\ 
			
			ADDA \cite{adda}        & 76.0$\pm$1.8 & 89.4$\pm$0.2 & 90.1$\pm$0.8 & 85.2 \\ 
			
			SBADA-GAN \cite{sbada_gan} & 76.1 & \textbf{97.6} & 95.0 & 89.6 \\ 
			
			RAAN \cite{raan}        & 89.2 & 89.0 & 92.1 & 90.1 \\ 
			
			ADR \cite{adr}          & 94.1$\pm$1.37 & 91.3$\pm$0.65 & 91.5$\pm$3.61 & 92.3 \\ 
			
			TPN \cite{tpn}          & 93.0 & 92.1 & 94.1 & 93.1 \\ 
			
			CyCADA \cite{cycada}    & 90.4$\pm$0.4 & 95.6$\pm$0.2 & 96.5$\pm$0.1 & 94.2 \\ 
			
			MCD \cite{mcd}          & \textbf{96.2}$\pm$0.4 & 94.2$\pm$0.7 & 94.1$\pm$0.3 & 94.8 \\ 
			
			CADA \cite{cons_ada}        & 90.9$\pm$0.2 & 96.4$\pm$0.1 & \textbf{97.0}$\pm$0.1 & 94.8 \\ 
			
			DAN \cite{dan}          & 71.1 & - & - & - \\ 
			
			CoGAN \cite{cogan}      & - & 91.2$\pm$0.8 & 89.1$\pm$0.8 & - \\ 
			
			DSN \cite{dsn}         & 82.7 & 91.3 & - & - \\ 
			
			LDC \cite{layerWiseCorrection} & 89.5$\pm$2.1 & - & - & - \\ 
			
			MSTN \cite{mstn}        & 91.7$\pm$1.5 & 92.9$\pm$1.1 & - & - \\ 
			
			PFAN \cite{pfan}        & 93.9$\pm$0.8 & 95.0$\pm$1.3 & - & - \\ 
			
			JDDA-C \cite{jdda}      & 94.2$\pm$0.1 & - & 96.7$\pm$0.1 & - \\
			
			MECA \cite{meca}        & 95.2 & - & - & - \\
			
			ASSC \cite{associativeDA}        & 95.7$\pm$1.5 & - & - & - \\ 
			
			\hline
			\textbf{DADA}           & 95.6 $\pm$ 0.5 & 96.1$\pm$0.4 & 96.5$\pm$0.2 & \textbf{96.1} \\
			\hline
		\end{tabular} 
	\end{center}
\end{table*}

\begin{figure}[t]
	\begin{center} 
		
		\subfigure[\textbf{A}$\rightarrow$\textbf{D}]{\includegraphics[width=1.0\linewidth]{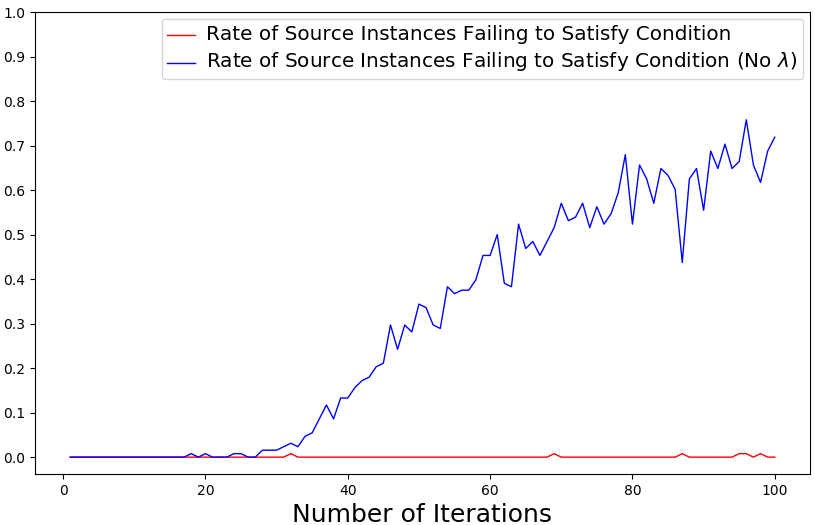}}
		\subfigure[\textbf{D}$\rightarrow$\textbf{A}]{\includegraphics[width=1.0\linewidth]{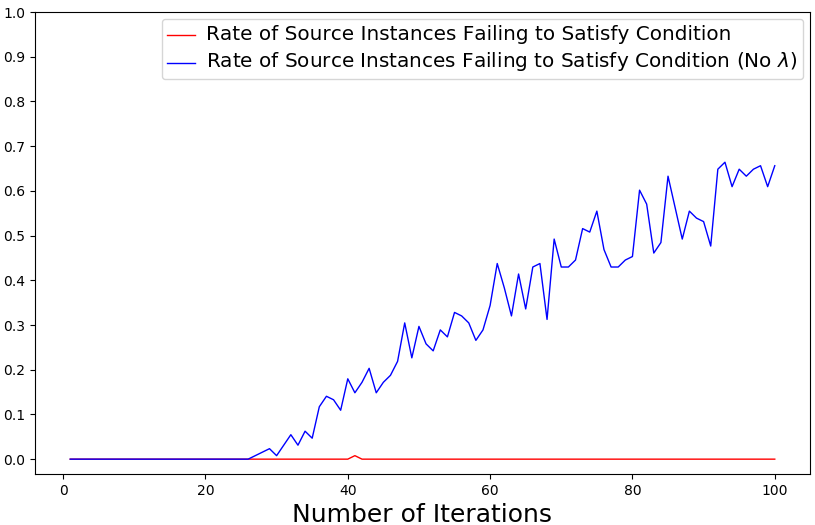}}
	\end{center}
	\caption{An illustration for the effect of the $\lambda$ on the rate of source instances failing to satisfy the condition in the early stage (e.g., the first $100$ iterations) of adversarial training on the two adaptation settings of (a) \textbf{A}$\rightarrow$\textbf{D} and (b) \textbf{D}$\rightarrow$\textbf{A}.}
	\label{fig:verify_lambda}
\end{figure}

\subsubsection{Alternative Choice of Adversarial Loss for Target Instances}
\label{sec7}
For a target adversarial loss, when maximized over the feature extractor $G(\cdot)$, we have an alternative choice. In this section, we give further discussion and experiments to compare our used ${\cal{L}}_G^t$ in loss (4) in the paper with this alternative. 

\begin{figure}[!ht]
	\begin{center} 
		\includegraphics[width=1.0\linewidth]{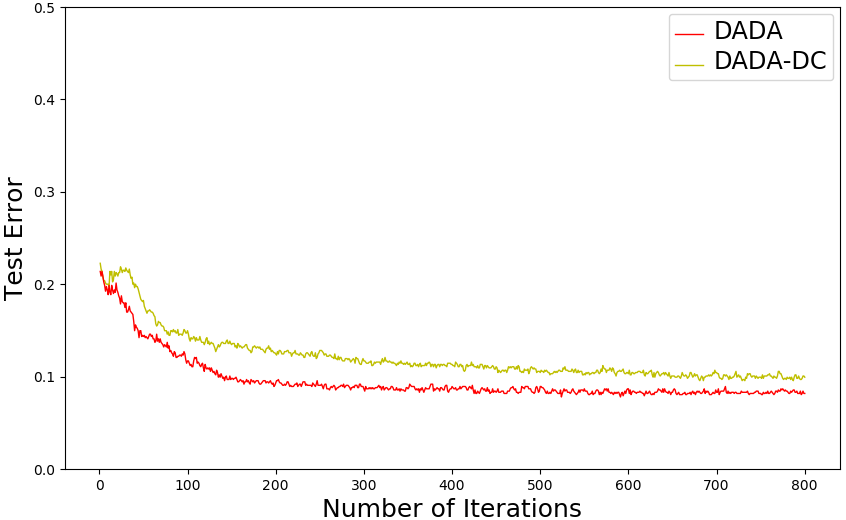}
	\end{center}
	\caption{Convergence performance in terms of test error on the adaptation setting \textbf{A} $\rightarrow$ \textbf{W}. Note that here we only show the convergence performance during adversarial training.}
	\label{fig:convergence_fg_variants}
\end{figure}

\begin{figure*}[ht]
	\includegraphics[scale=0.28]{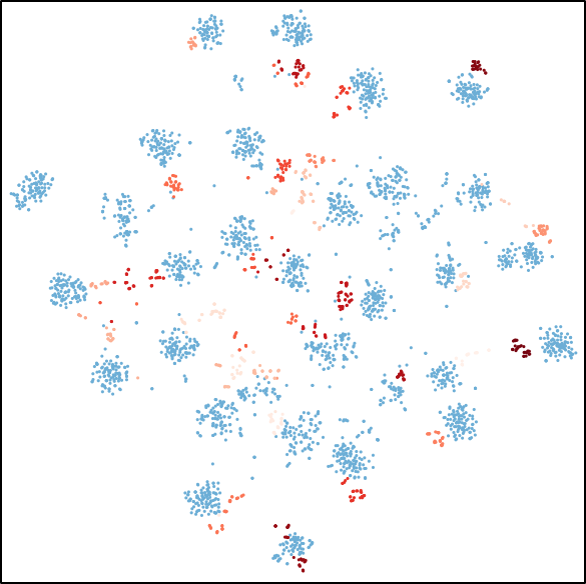} \hfill  \includegraphics[scale=0.28]{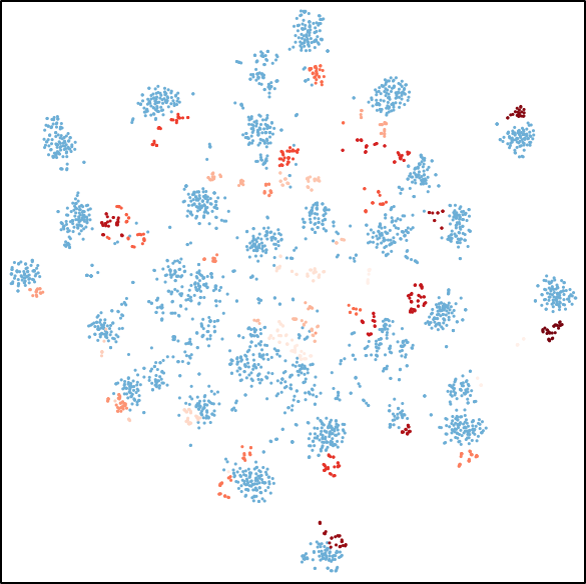} \hfill \includegraphics[scale=0.28]{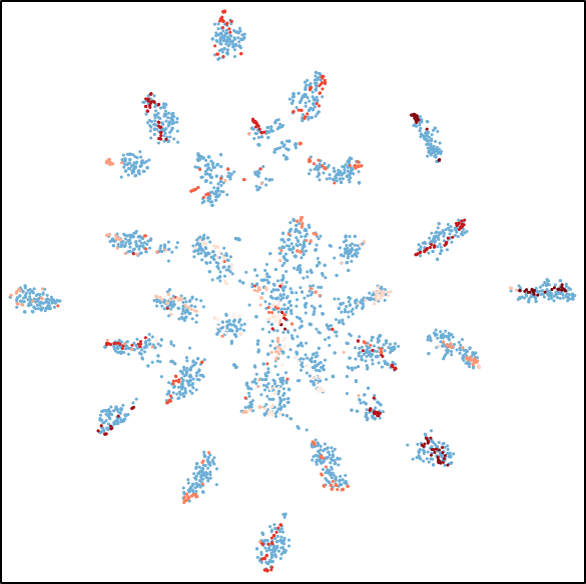} \hfill  \includegraphics[scale=0.28]{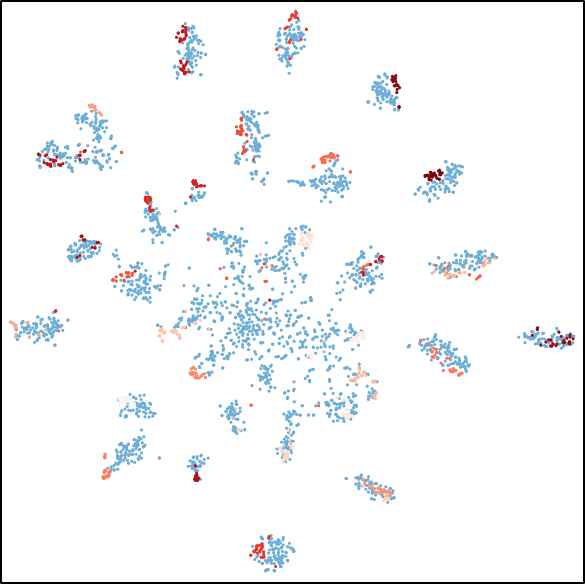} \\
	\caption{ The t-SNE visualization of feature alignment between the source (blue) and target (red) domains by No Adaptation, DANN, DANN-CA, and DADA (from left to right). Samples of plotting are from the adaptation setting of \textbf{A} $\rightarrow$ \textbf{W} in Table 1 in the paper. Note that different degrees of the red color indicate different target categories.}  
	\label{fig:t_sne}  
\end{figure*}

Inspired by the works \cite{simultaneous_transfer,symnets}, one may opt for a symmetric adversarial loss
\begin{eqnarray}\label{EqnAdversarialConfusionLossOnTarget}
\begin{aligned}
{\cal{L}}_{G}^t (G, F) =  \frac{1}{n_t} \sum_{j=1}^{n_t} \sum_{k=1}^K \bar{p}_k(\mathbf{x}_j^t)[&\frac{1}{2} \log p_{K+1}(\mathbf{x}_j^t) + \\ &\frac{1}{2}\! \log(1\!-\!p_{K+1}(\mathbf{x}_j^t))] ,
\end{aligned}
\end{eqnarray}
which when maximized over $G(\cdot)$, gives a confused prediction of $p_{K+1}(\mathbf{x}^t) = 0.5$. This result does not give category prediction $p_{y^t}(\mathbf{x}^t)$ on the \emph{unknown} true category $y^t$ of a target instance $\mathbf{x}^t$ a chance to approach $1$. Thus, this alternative choice is sub-optimal. 

In contrast, our used ${\cal{L}}_G^t$ in loss (4) in the paper gives a prediction of $p_{K+1}(\mathbf{x}^t) = 0$ when maximized over $G(\cdot)$. This result gives $p_{y^t}(\mathbf{x}^t)$ a better chance to approach $1$, i.e. $\bar{p}_{y^t}(\mathbf{x}^t)$ is more likely to approach $1$. In other words, the target data are more likely to be correctly classified, which is enabled by our proposed \emph{mutually inhibitory relation between the category and domain predictions}. 

To compare the effectiveness of our used ${\cal{L}}_G^t$ in loss (4) in the paper and this alternative choice, we conduct experiments on Office-31 \cite{office31} based on ResNet-50 \cite{resnet}, by replacing ${\cal{L}}_G^t$ in loss (4) in the paper with the domain confusion loss (\ref{EqnAdversarialConfusionLossOnTarget}) in our main minimax probem (7) in the paper. 
We denote these this alternative as ``DADA-DC". Results in Table \ref{table:other_variants_office31} and convergence performances in Figure \ref{fig:convergence_fg_variants} show advantages of our used ${\cal{L}}_G^t$ in loss (4) in the paper.

\subsubsection{Feature Visualization}
To visualize how different methods are effective at aligning learned features on the source and target domains, we use t-SNE embeddings \cite{t_sne} to plot the output activations from the feature extractors of ``No Adaptation'', DANN, DANN-CA, and DADA. Figure \ref{fig:t_sne} gives the plotting, where samples are from the adaptation setting of \textbf{A} $\rightarrow$ \textbf{W} of Office-31 \cite{office31} based on ResNet-50 \cite{resnet}. Figure \ref{fig:t_sne} shows qualitative improvements of these methods at aligning features across data domains, i.e., the distribution of target samples (red) changes from the scattered state of DANN to multiple category-wise clusters of DADA, which are aligned with source samples (blue) of corresponding categories. Note that for the source domain, since we aim to achieve a balance between the transferability and discriminability \cite{tat}, the categories are not perfectly separated. Since the transferability is enhanced, the target categories are well separated.

\subsubsection{Digits}
To validate the robustness of our proposed DADA, we evaluate different methods on Digits datasets of MNIST \cite{mnist}, SVHN \cite{svhn}, and USPS \cite{usps} based on modified LeNet in Table \ref{table:results_digits}. Note that results of existing methods are quoted from their respective papers or the recent works \cite{adr,mcd}. We follow these methods and report accuracies on the target test data in the format of $\texttt{mean}\pm\texttt{std}$ over five random trials. Our proposed DADA consistently achieves a good result on different adaptation settings, showing its excellent robustness.

\begin{table*}[!htb]
	\begin{center}
		\caption{Results for partial domain adaptation on Office-31 based on ResNet-50.}
		\label{table:results_office31_partial_transfer}
		\begin{tabular}{lccccccc}
			\hline
			Methods  & A $\rightarrow$ W & D $\rightarrow$ W & W $\rightarrow$ D & A $\rightarrow$ D & D $\rightarrow$ A & W $\rightarrow$ A & Avg \\
			\hline
			No Adaptation \cite{resnet} & 54.52 & 94.57 & 94.27 & 65.61 & 73.17 & 71.71 & 75.64 \\ 
			
			DAN \cite{dan}       & 46.44 & 53.56 & 58.60 & 42.68 & 65.66 & 65.34 & 55.38 \\ 
			
			DANN \cite{dann}     & 41.35 & 46.78 & 38.85 & 41.36 & 41.34 & 44.68 & 42.39 \\ 
			
			ADDA \cite{adda}     & 43.65 & 46.48 & 40.12 & 43.66 & 42.76 & 45.95 & 43.77 \\ 
			
			RTN \cite{rtn}       & 75.25 & 97.12 & 98.32 & 66.88 & 85.59 & 85.70 & 84.81 \\ 
			
			JAN \cite{jan}       & 43.39 & 53.56 & 41.40 & 35.67 & 51.04 & 51.57 & 46.11 \\ 
			
			Luo \emph{et al.} \cite{lel}   & 73.22 & 93.90 & 96.82 & 76.43 & 83.62 & 84.76 & 84.79 \\ 
			
			PADA \cite{pada}     & 86.54 & 99.32 & \textbf{100.00} & 82.17 & 92.69 & \textbf{95.41} & 92.69 \\ 
			\hline
			\textbf{DADA-P}    & \textbf{90.73} & \textbf{100.00} & \textbf{100.00} & \textbf{87.90} & \textbf{94.71} & 94.89 & \textbf{94.71} \\ 
			\hline
		\end{tabular}
	\end{center}
\end{table*}

\subsection{Partial Domain Adaptation. }

For each partial adaptation setting of Office-31, Office-Home, and ImageNet-Caltech, we follow the work \cite{pada} to report the mean classification result on the target domain over three random trials.

\subsubsection{Office-31}
We compare in Table \ref{table:results_office31_partial_transfer} our proposed method with existing ones on Office-31 based on ResNet-50 \cite{resnet} pre-trained on ImageNet \cite{imagenet}. Results of existing methods are quoted from PADA \cite{pada}. Our proposed DADA-P outperforms all comparative methods by a large margin, showing the effectiveness of the adopted category-level weighting mechanism on reducing the negative influence of source outliers on adaptation settings with small source domain and small target domain, e.g., \textbf{A} $\rightarrow$ \textbf{W}. Although PADA uses the same weighting mechanism, it performs much worse than our proposed DADA-P, suggesting the effectiveness of DADA-P on enhancing the positive influence of shared categories.

From the experimental results, several interesting observations can be derived. \textbf{(1)} Previous deep domain adaptation methods including those based on domain-adversarial training (e.g., DANN) and those based on MMD (e.g., DAN) perform much worse than the very baseline ``No Adaptation'', showing the huge impact of negative transfer. Domain-adversarial training based methods aim to learn domain-invariant intermediate features to deceive the domain classifier, and MMD based methods aim to minimize the discrepancy between data distributions of the source and target domains. Both of them align the whole source domain to the whole target one. However, in partial domain adaptation, since the source domain contains categories that do not exist in the target domain, i.e., outlier source categories, they will suffer false alignment between the outlier source categories and the target
domain. This explains their poor performance in partial domain adaptation. \textbf{(2)} Among previous deep domain adaptation methods, RTN is the only one that performs better than ``No Adaptation''. RTN exploits the entropy minimization principle \cite{em} to encourage the low-density separation of target categories. Its target task classifier directly has access to the unlabeled target data and can amend itself to pass through the target low-density regions where the outlier source categories may exist, which alleviate the negative influence of source outliers. Nevertheless, PADA, which does not use the entropy minimization principle but a category-level weighting mechanism, performs much better than RTN, demonstrating that RTN still suffers negative transfer and may be not able to bridge such a large domain discrepancy caused by different label spaces. \textbf{(3)} Although our proposed DADA-P applies the same weighting mechanism as PADA, it performs much better than PADA. PADA has a separate design of task and domain classifiers and only aims to align marginal feature distributions, whereas our proposed DADA-P based on an integrated domain and task classifier, aims to promote the joint distribution alignment across domains. This explains the good performance of our proposed method in partial domain adaptation.

\begin{figure}[t]
	\begin{center}
		\includegraphics[width=1.0\linewidth]{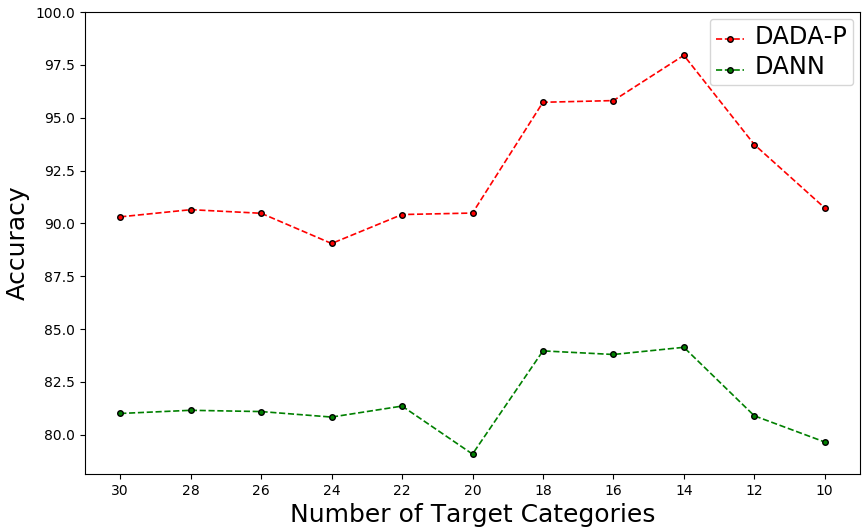}
	\end{center}
	\caption{The accuracy curve of varying the number of target categories for the baseline DANN \cite{dann} and our proposed DADA-P on the partial adaptation setting \textbf{A} $\rightarrow$ \textbf{W} of Office-31 with a base network of ResNet-50.}
	\label{fig:partial_vary_tarC}
\end{figure}

\begin{figure}[t]
	\begin{center}
		\includegraphics[width=1.0\linewidth]{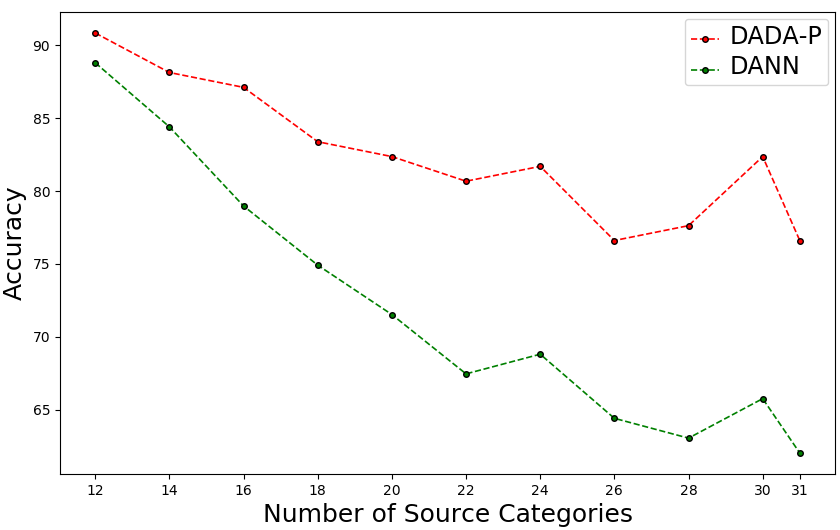}
	\end{center}
	\caption{The accuracy curve of varying the number of source categories for the baseline DANN \cite{dann} and our proposed DADA-P on the partial adaptation setting \textbf{A} $\rightarrow$ \textbf{W} of Office-31 with a base network of AlexNet.}
	\label{fig:partial_vary_srcC}
\end{figure}

\begin{table*}[!htb]
	\begin{center}
		\caption{Results for partial domain adaptation on Office-31 based on AlexNet.}
		\label{table:results_office31_partial_transfer_alexnet}
		\resizebox{1.0\textwidth}{!}{
		\begin{tabular}{lccccccc}
			\hline
			Methods  & A $\rightarrow$ W & D $\rightarrow$ W & W $\rightarrow$ D & A $\rightarrow$ D & D $\rightarrow$ A & W $\rightarrow$ A & Avg \\
			\hline
			No Adaptation \cite{alexnet} & 58.51 & 95.05 & 98.08 & 71.23 & 70.60 & 67.74 & 76.87 \\ 
			
			DAN \cite{dan}               & 56.52 & 71.86 & 86.78 & 51.86 & 50.42 & 52.29 & 61.62 \\ 
			
			DANN \cite{dann}             & 49.49 & 93.55 & 90.44 & 49.68 & 46.72 & 48.81 & 63.11 \\ 
			
			ADDA \cite{adda}             & 70.68 & 96.44 & 98.65 & 72.90 & 74.26 & 75.56 & 81.42 \\ 
			
			RTN \cite{rtn}               & 66.78 & 86.77 & 99.36 & 70.06 & 73.52 & 76.41 & 78.82 \\ 
			
			SAN \cite{san}               & \textbf{80.02} & 98.64 & \textbf{100.00} & 81.28 & 80.58 & 83.09 & 87.27 \\ 
			
			Zhang \emph{et al.} \cite{iwan} & 76.27 & \textbf{98.98} & \textbf{100.00} & 78.98 & 89.46 & 81.73 & 87.57 \\ 
			\hline
			\textbf{DADA-P}    & 76.61 & \textbf{98.98} & \textbf{100.00} & \textbf{85.56} & \textbf{93.81} & \textbf{93.28} & \textbf{91.37} \\ 
			\hline
		\end{tabular}
	}
	\end{center}
\end{table*}

To investigate a wider spectrum of partial domain adaptation, we conduct experiments by varying the number of target categories. Figure \ref{fig:partial_vary_tarC} shows results for the baseline DANN \cite{dann} and our proposed DADA-P on the partial adaptation setting \textbf{A} $\rightarrow$ \textbf{W} of Office-31 with a base network of ResNet-50.  The source domain has always $31$ categories, but the number of target categories varies from $30$ to $10$, i.e., $\{30, 28, 26, 24, 22, 20, 18, 16, 14, 12, 10\}$. As the number of target categories decreases, performances of the two methods have no evident decline in spite of the aggravation of negative transfer effect, since the difficulty of domain adaptation problem itself becomes smaller. We observe a sharp rise and a dramatic drop when the number of target categories decreases from $20$ to $18$ and from $14$ to $12$ respectively. One explanation is that the positive influence incurred by reducing the difficulty of domain adaptation problem itself is more (for the former observation) or less (for the latter one) than the negative influence caused by increasing the domain discrepancy. The results show that our proposed DADA-P performs much better than DANN in all settings. It is noteworthy that the relative performance improvement becomes larger when the number of target categories decreases, testifying the superiority of our methods in reducing the influence of negative transfer. Thus, given a source domain, our methods can perform much better when applied to the target domain with unknown number of categories.

\begin{table*}[t]
	\begin{center}
		\caption{Results for partial domain adaptation on Office-Home based on ResNet-50.}
		\label{table:results_officehome_partial_transfer}
		\resizebox{1.0\textwidth}{!}{
		\begin{tabular}{lccccccccccccc}
			\hline
			Methods  & Ar$\rightarrow$Cl & Ar$\rightarrow$Pr & Ar$\rightarrow$Rw & Cl$\rightarrow$Ar & Cl$\rightarrow$Pr & Cl$\rightarrow$Rw & Pr$\rightarrow$Ar & Pr$\rightarrow$Cl & Pr$\rightarrow$Rw & Rw$\rightarrow$Ar & Rw$\rightarrow$Cl & Rw$\rightarrow$Pr    & Avg  \\
			\hline
			No Adaptation \cite{resnet} & 38.57 & 60.78 & 75.21 & 39.94 & 48.12 & 52.90 & 49.68 & 30.91 & 70.79 & 65.38 & 41.79 & 70.42 & 53.71 \\
			
			DAN \cite{dan}       & 44.36 & 61.79 & 74.49 & 41.78 & 45.21 & 54.11 & 46.92 & 38.14 & 68.42 & 64.37 & 45.37 & 68.85 & 54.48 \\
			
			DANN \cite{dann}     & 44.89 & 54.06 & 68.97 & 36.27 & 34.34 & 45.22 & 44.08 & 38.03 & 68.69 & 52.98 & 34.68 & 46.50 & 47.39 \\
			
			RTN \cite{rtn}       & 49.37 & 64.33 & 76.19 & 47.56 & 51.74 & 57.67 & 50.38 & 41.45 & 75.53 & 70.17 & 51.82 & 74.78 & 59.25 \\  
			
			PADA \cite{pada}     & 51.95 & 67.00 & 78.74 & 52.16 & 53.78 & 59.03 & 52.61 & 43.22 & 78.79 & 73.73 & \textbf{56.60} & 77.09 & 62.06 \\
			\hline
			\textbf{DADA-P}    & \textbf{52.92} & \textbf{82.54} & \textbf{86.78} & \textbf{71.23} & \textbf{69.75} & \textbf{76.72} & \textbf{73.06} & \textbf{52.84} & \textbf{85.90} & \textbf{77.69} & 56.50 & \textbf{85.98} & \textbf{72.66} \\
			\hline
		\end{tabular}
	}
	\end{center}
\end{table*}

\begin{table}[t]
	\begin{center}
		\caption{Results for partial domain adaptation on ImageNet-Caltech based on ResNet-50.}
		\label{table:results_imagenet_caltech_partial_transfer}
		\begin{tabular}{lccc}
			\hline
			Methods  & I$\rightarrow$C & C$\rightarrow$I & Avg  \\
			\hline
			No Adaptation \cite{resnet} & 71.65 & 66.14 & 68.90 \\
			
			DAN \cite{dan}              & 71.57 & 66.48 & 69.03  \\
			
			DANN \cite{dann}            & 68.67 & 52.97 & 60.82  \\
			
			RTN \cite{rtn}              & 72.24 & 68.33 & 70.29  \\  
			
			PADA \cite{pada}            & 75.03 & 70.48 & 72.76  \\
			\hline
			\textbf{DADA-P}           & \textbf{80.94} & \textbf{76.91} & \textbf{78.93} \\
			\hline
		\end{tabular}
	\end{center}
\end{table}

We compare in Table \ref{table:results_office31_partial_transfer_alexnet} our proposed method with existing ones on Office-31 based on AlexNet \cite{alexnet} pre-trained on ImageNet. Results of existing methods are quoted from their respective papers or SAN \cite{san}. Our proposed DADA-P achieves a much better result than all comparative methods, showing the efficacy of our methods with a shallower neuron network as the base network.

To investigate the influence of the number of outlier source categories on the performance, we conduct experiments by varying the number of source categories. Figure \ref{fig:partial_vary_srcC} shows results for the baseline DANN \cite{dann} and our proposed DADA-P on the partial adaptation setting \textbf{A} $\rightarrow$ \textbf{W} of Office-31 with a base network of AlexNet. The target domain has always $10$ categories, but the number of source categories varies from $12$ to $31$, i.e., $\{12, 14, 16, 18, 20, 22, 24, 26, 28, 30, 31\}$. As the number of source categories increases, performances of the two methods have evident decline but also some rises, e.g., when the number of source categories increases from $22$ to $24$ and from $28$ to $30$. One explanation is that the positive influence incurred by increasing discriminative information of categories, especially those related to the target domain, is more than the negative influence caused by increasing the domain discrepancy. The results show that our proposed DADA-P significantly outperforms DANN in all settings. Particularly, the relative performance improvement is larger when the number of source categories is larger, demonstrating that our methods are more robust to the number of outlier source categories. Thus, for a given target task, our methods can have a much better performance when utilizing different source tasks.

\begin{table*}[th]
	\begin{center}
		\caption{Results for open set domain adaptation on Office-31 based on AlexNet. Note that all methods do not use unknown source instances. \emph{OS*} indicates the mean classification result over known categories whereas \emph{OS} also includes the unknown category.}
		\label{table:results_office31_open_set_alexnet}
		\resizebox{1.0\textwidth}{!}{
		\begin{tabular}{lcccccccccccccc}
			\hline
			Methods  & \multicolumn{2}{c}{A $\rightarrow$ W} & \multicolumn{2}{c}{D $\rightarrow$ W} & \multicolumn{2}{c}{W $\rightarrow$ D} & \multicolumn{2}{c}{A $\rightarrow$ D} & \multicolumn{2}{c}{D $\rightarrow$ A} & \multicolumn{2}{c}{W $\rightarrow$ A} & \multicolumn{2}{c}{Avg} \\
			& OS & OS* & OS & OS* & OS & OS* & OS & OS* & OS & OS* & OS & OS* & OS & OS*  \\
			\hline
			No Adaptation \cite{alexnet} & 57.1 & 55.0 & 44.1 & 39.3 & 62.5 & 59.2 & 59.6 & 59.1 & 14.3 & 5.9 & 13.0 & 4.5 & 40.6 & 37.1  \\ 
			
			DAN \cite{dan}               & 41.5 & 36.2 & 34.4 & 28.4 & 62.0 & 58.5 & 47.8 & 44.3 & 9.9 & 0.9 & 11.5 & 2.7 & 34.5 & 28.5  \\ 
			
			DANN \cite{dann}             & 31.0 & 24.3 & 33.6 & 27.3 & 49.7 & 44.8 & 40.8 & 35.6 & 10.4 & 1.5 & 11.5 & 2.7 & 29.5 & 22.7  \\ 
			
			ATI-$\lambda$ \cite{openSetDA} & 65.3 & -  & 82.2 & -  & 92.7 & -  & 72.0 & -  & 66.4 & - & 71.6 & -  & 75.0 & -  \\
			
			AODA \cite{bp_for_os} & 70.1 & 69.1 & \textbf{94.4} & \textbf{94.6} & \textbf{96.8} & \textbf{96.9} & 76.6 & 76.4 & 62.5 & 62.3 & \textbf{82.3} & \textbf{82.2} & 80.4 & 80.2 \\ 
			\hline
			\textbf{DADA-O}    & \textbf{75.5} & \textbf{75.6} & 91.2 & 93.0 & 93.3 & 94.4 & \textbf{82.7} & \textbf{83.9} & \textbf{73.5} & \textbf{74.8} & 71.1 & 71.6 & \textbf{81.2} & \textbf{82.2} \\ 
			\hline
		\end{tabular}
	}
	\end{center}
\end{table*}

\subsubsection{Office-Home}
We compare in Table \ref{table:results_officehome_partial_transfer} our proposed method with existing ones on Office-Home based on ResNet-50. Results of existing methods are quoted from PADA \cite{pada}. Our proposed DADA-P significantly outperforms all comparative methods, showing the efficacy of DADA-P on adaptation settings with more categories in both the source and target domains and larger domain discrepancy between the two domains, e.g., \textbf{Cl} $\rightarrow$ \textbf{Rw}. 

\subsubsection{ImageNet-Caltech}
We compare in Table \ref{table:results_imagenet_caltech_partial_transfer} our proposed method with existing ones on ImageNet-Caltech based on ResNet-50. Results of existing methods are quoted from PADA \cite{pada}. Our proposed DADA-P outperforms all comparative methods by a large margin, showing the effectiveness of DADA-P on adaptation settings with large-scale source and target domains and a large number of categories in the two domains. 

\subsection{Open Set Domain Adaptation. }We compare in Table \ref{table:results_office31_open_set_alexnet} our proposed method with existing ones on Office-31 based on AlexNet \cite{alexnet} pre-trained on ImageNet \cite{imagenet}. Results of existing methods are quoted from AODA \cite{bp_for_os}. Our proposed DADA-O outperforms all comparative methods in both evaluation metrics of OS* and OS, showing the efficacy of DADA-O in both aligning distributions of the known instances across domains and identifying the unknown target instances as the unknown category for open set domain adaptation. 

From the experimental results, we have some interesting observations. \textbf{(1)} DAN and DANN perform much worse than ``No Adaptation''. DAN and DANN aim to align the whole marginal feature distributions across the source and target domains. If the target domain contains unknown instances, false alignment between the known source instances and unknown target ones will occur, resulting in a sharp drop of the classification performance. \textbf{(2)} DANN performs worse than DAN, since DANN is better at aligning marginal feature distributions across data domains, leading to more serious false alignment. \textbf{(3)} ATI-$\lambda$ and AODA can effectively reduce false alignment, since they have a good outlier rejection mechanism to recognize the unknown instances. \textbf{(4)} The results of all comparative methods on almost all adaptation settings are better in the evaluation metric OS than OS*, showing that many known target instances are classified as the unknown category. Since Open-set SVM is trained to detect outliers and the task classifier of AODA is trained to recognize all the target instances as the unknown category, they are inclined to classify the target instances as the unknown category. \textbf{(5)} For our proposed DADA-O, the results of all adaptation settings are better in the evaluation metric OS* than OS, since their classifiers are trained to classify all target instance as the unknown category with a small probability $q$, which can minimize the misclassification of the known target instances as the unknown category.

\begin{figure}[ht]
	\begin{center}
		\includegraphics[width=1.0\linewidth]{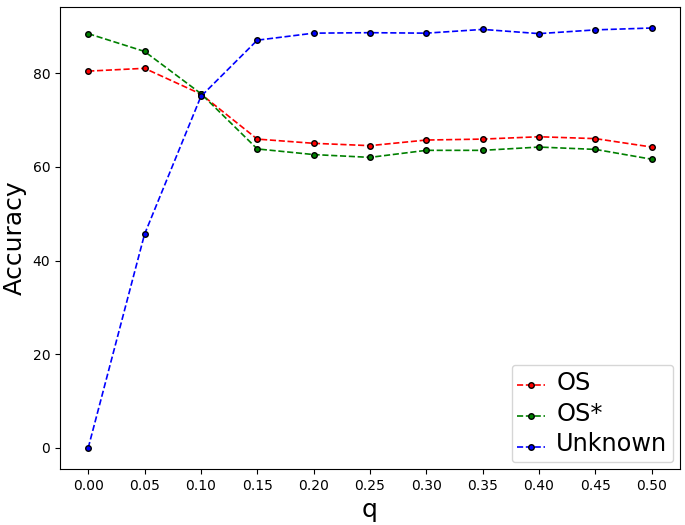}
	\end{center}
	\caption{The accuracy curve of varying $q$ for our proposed DADA-O on the open set adaptation setting \textbf{A} $\rightarrow$ \textbf{W} of Office-31 with a base network of AlexNet. The accuracy for unknown target instances is denoted by the blue line.}
	\label{fig:open_set_vary_t}
\end{figure}

\begin{figure*}[!t]
	\begin{center} 
		
		\includegraphics[width=0.85\linewidth]{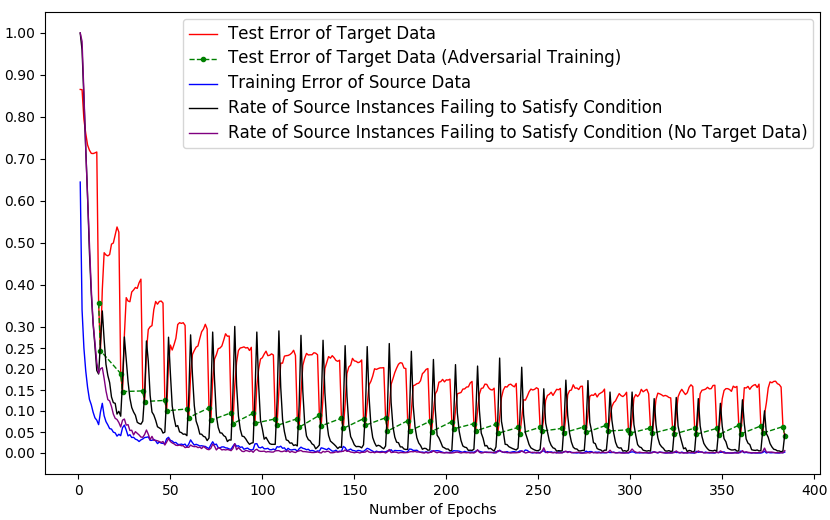} \\
		(a) \textbf{MNIST}$\rightarrow$\textbf{USPS} ($N_{alter}=32$) \\
		\includegraphics[width=0.85\linewidth]{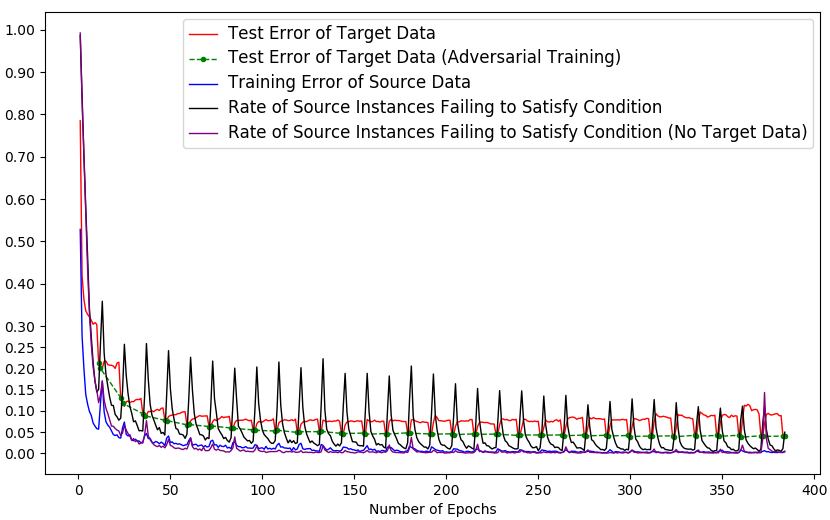} \\
		(b) \textbf{USPS}$\rightarrow$\textbf{MNIST} ($N_{alter}=32$) \\
	\end{center}
	\caption{Training processes in terms of the test error of the target data for each epoch, the test error of the target data for each epoch of adversarial training, the training error of the source data for each epoch, the rate of source instances failing to satisfy the condition for each epoch, and the rate of source instances failing to satisfy the condition for each epoch when no target data is used in the adversarial training, on the two adaptation settings of (a) \textbf{MNIST}$\rightarrow$\textbf{USPS} and (b) \textbf{USPS}$\rightarrow$\textbf{MNIST}.}
	\label{fig:training_process}
\end{figure*} 

To investigate the influence of $q$ on the performance, we conduct experiments by varying $q$. Figure \ref{fig:open_set_vary_t} shows results for our proposed DADA on the open set adaptation setting \textbf{A} $\rightarrow$ \textbf{W} of Office-31 with a base network of AlexNet. As $q$ increases, accuracies of OS and OS* decrease and the accuracy of Unknown increases, which means that the target instances are more likely classified as the unknown category. This confirms the statements we present in Section \textbf{Extension for Open Set Domain Adaptation} in the paper. When $q=0$, the objective of the feature extractor is to align the whole source domain and the whole target domain, resulting in the misclassification of all unknown target instances as the known categories, as illustrated in Figure \ref{fig:open_set_vary_t}. This demonstrates that the model does not learn feature representations that can separate the unknown target instances from the known instances. To make a trade-off, we empirically set $q=0.1$ for all open set adaptation settings.

\section{Investigation for Our Used Training Scheme} 
\label{sec5}
In this section, we investigate our used training scheme of pre-training DADA on the labeled source data and maintaining the same supervision signal in the adversarial training of DADA, on benchmark datasets of MNIST \cite{mnist} and USPS \cite{usps}, where two adaptation settings of \textbf{MNIST}$\rightarrow$\textbf{USPS} and \textbf{USPS}$\rightarrow$\textbf{MNIST} are built.

To always satisfy the condition of $p_{y^s}^s>0.5$ discussed in Section \textbf{Discriminative Adversarial Learning} in the paper, we train DADA of $F(G(\cdot))$ by a well-designed scheme, which can be formulated as alternating the classification training on the labeled source data and the adversarial training of DADA on the labeled source data and unlabeled target data. We denote the number of training epochs or training iterations for classification training in each alternation respectively as $T_{cls}$ and ${\hat{T}}_{cls}$, the number of training epochs or training iterations for adversarial training in each alternation respectively as $T_{adv}$ and ${\hat{T}}_{adv}$, and the number of alternating the classification training and adversarial training as $N_{alter}$. 
For the two adaptation settings of \textbf{MNIST}$\rightarrow$\textbf{USPS} and \textbf{USPS}$\rightarrow$\textbf{MNIST}, $T_{cls}$, $T_{adv}$, and $N_{alter}$ are respectively set to $10$, $2$, and $16$, according to the rate of source instances failing to satisfy the condition; the hyper-parameter $\lambda$ (cf. Section \textbf{Discriminative Adversarial Learning} in the paper for its definition) is not used, since $T_{adv}$ is a quite small number. We investigate the efficacy of our used training scheme on keeping the condition satisfied by visualizing training processes on the two adaptation settings in Figure \ref{fig:training_process}. 

From Figure \ref{fig:training_process}, we can obtain several interesting observations. (1) The classification training makes ``Rate of Source Instances Failing to Satisfy Condition'' fall into a valley whereas the adversarial training of DADA makes it rise to a peak, showing that a part of source instances change from satisfying the condition to not satisfying it during adversarial training. (2)  ``Rate of Source Instances Failing to Satisfy Condition (No Target Data)'' is much lower than ``Rate of Source Instances Failing to Satisfy Condition'' at epochs of adversarial training, showing that the training of target data affects the source data and results in that a part of them do not satisfy the condition. (3) ``Rate of Source Instances Failing to Satisfy Condition'' declines to a very low value in an oscillatory manner, showing the efficacy of this training scheme on keeping the condition satisfied. (4) ``Training Error of Source Data'' is low at epochs of adversarial training, showing that our proposed DADA has the same effect as classification training. (5) All valleys of ``Test Error of Target Data'' are derived from the adversarial training of DADA, showing the excellent efficacy of our proposed DADA in aligning the source and target domains. (6) At epochs of adversarial training, the lower ``Rate of Source Instances Failing to Satisfy Condition'' is, the more improvement of performance is obtained, showing the necessity of satisfying the condition. (6) The good performances of DADA on the two adaptation settings of \textbf{MNIST}$\rightarrow$\textbf{USPS} and \textbf{USPS}$\rightarrow$\textbf{MNIST}, which are very close to the perfect performance of $100\%$, confirm the efficacy of our proposed DADA in aligning the joint distributions of feature and category across the two domains.

For each closed set adaptation setting of Office-31, $T_{cls}$, ${\hat{T}}_{adv}$, and $N_{alter}$ are respectively set to $200$, $800$, and $1$. For the closed set adaptation setting of Syn2Real, ${\hat{T}}_{cls}$, ${\hat{T}}_{adv}$, and $N_{alter}$ are respectively set to $2000$, $1000$, $1$. For all these adaptation settings, the hyper-parameter $\lambda$ is used. 

\end{document}